\documentclass{article}

\usepackage[preprint]{neurips_2026}

\usepackage{amsmath,amssymb,amsthm,mathtools}
\usepackage{amsfonts,bm}
\usepackage{algorithm}
\usepackage{algpseudocode}
\usepackage{booktabs}
\usepackage{multirow}
\usepackage{graphicx}
\usepackage{subcaption}
\usepackage{float}
\usepackage{xcolor}
\usepackage{enumitem}
\usepackage[hidelinks]{hyperref}
\usepackage{tikz}

\usetikzlibrary{arrows.meta,positioning,fit,calc}

\newtheorem{definition}{Definition}

\newtheorem{remark}{Remark}
\newtheorem{proposition}{Proposition}
\newtheorem{lemma}{Lemma}
\newtheorem{theorem}{Theorem}

\newcommand{\norm}[1]{\left\lVert #1 \right\rVert}
\newcommand{\inner}[2]{\left\langle #1,#2 \right\rangle}

\providecommand{\keywords}[1]{%
  \par\medskip
  \noindent\textbf{Keywords: } #1
}

\title{SMA-DP: Spectral Memory-Aware Differential Privacy for Deep Learning}

\author{%
  Mohammad Partohaghighi \\
  Department of Electrical Engineering and Computer Science \\
  University of California, Merced \\
  Merced, CA, USA \\
  \texttt{mpartohaghighi@ucmerced.edu} \\
  \And
  Roummel P. Marcia \\
  Department of Applied Mathematics \\
  University of California, Merced \\
  Merced, CA, USA \\
  \texttt{rmarcia@ucmerced.edu} \\
}

\begin{document}

\maketitle

\begin{abstract}
Differentially private stochastic gradient descent (DP-SGD) enables private
deep learning through per-example clipping and calibrated Gaussian noise, but
the resulting updates can be high-variance and may degrade utility, especially
on more challenging datasets. We propose \textbf{SMA-DP-SGD}, a
\textbf{Spectral Memory-Aware Differentially Private Stochastic Gradient
Descent} method that augments DP-SGD with a fractional memory branch constructed
only from previously privatized noisy releases. The method uses
WeightWatcher-inspired power-law spectral exponents as group-wise reliability
signals to adapt the decay and effective depth of memory; in our experiments,
the group-wise formulation is instantiated layer-wise. Private-history
alignment, norm matching, and warm-up activation further stabilize the memory
contribution. The privacy structure remains transparent: conditioned on the
private release history, the memory branch is fixed, and the only newly
data-dependent term is the current clipped sum scaled by a fixed coefficient
\(\beta\). Thus, SMA-DP-SGD preserves a clean conditional sensitivity structure
and exactly recovers standard group-wise DP-SGD when \(\beta=1\). Experiments
on CIFAR-100, CIFAR-10, and MNIST show that SMA-DP-SGD achieves competitive or
superior accuracy compared with several DP optimization baselines, with the
strongest gains appearing on the more challenging CIFAR-100 and CIFAR-10
settings. Ablations on CIFAR-10 show that the mixing parameter \(\beta\)
controls the privacy--utility trajectory, while spectral and memory diagnostics
confirm that the method maintains a controlled short-to-moderate effective
memory depth and a small memory-branch ratio. Runtime analysis shows that this
spectral memory mechanism introduces additional computational overhead, about
\(2.94\times\) DP-SGD in our CIFAR-10 implementation, highlighting a practical
trade-off between adaptive private memory and computational cost.
\end{abstract}



\section{Introduction}
\label{sec:introduction}

Deep neural networks are increasingly trained on sensitive data, including
medical records, genomic profiles, user behavior, financial transactions, and
personal-device data. Differential privacy (DP) provides a rigorous framework
for limiting the influence of any individual training example on the learned
model~\cite{dwork2006calibrating,dwork2014algorithmic,partohaghighi2026roughness2}. In private deep
learning, differentially private stochastic gradient descent (DP-SGD) has become
a standard workhorse: it clips per-example gradients and adds calibrated
Gaussian noise before updating the model~\cite{abadi2016deep}. This mechanism
has enabled private training in centralized, federated, and data-sensitive
learning settings~\cite{mcmahan2018learning,papernot2018scalable}. However,
the same clipping and noise that provide privacy also perturb the optimization
signal, often increasing variance, slowing convergence, and reducing utility.

A substantial body of work has improved private training through sharper
privacy accounting, subsampling analyses, adaptive clipping, transfer learning,
and scaling strategies~\cite{mironov2017renyi,wang2019subsampled,bu2020deep,
andrew2021differentially,tramer2021differentially,de2022unlocking}. These
developments have significantly improved the practicality of DP deep learning,
yet noisy private optimization remains challenging. In particular, DP-SGD
typically relies heavily on the current noisy clipped gradient. When the privacy
budget is tight, this current update can be a high-variance estimate of the
underlying optimization direction, making training sensitive to clipping norms,
noise multipliers, learning rates, and model scale~\cite{bassily2014private,
de2022unlocking}.

Private optimization need not be memoryless. Classical stochastic optimization
already exploits historical information through momentum, acceleration,
adaptive preconditioning, and adaptive moments~\cite{robbins1951stochastic,
polyak1964some,nesterov1983method,duchi2011adaptive,kingma2015adam}. The same
principle is compelling under DP: previous noisy private releases may still
contain useful directional information about the optimization trajectory.
Discarding all historical information forces training to rely primarily on the
current noisy clipped gradient, which can amplify variance and slow convergence.
At the same time, memory must be introduced carefully under DP. Raw historical
gradients remain data-dependent, and current-gradient-dependent memory gates can
create additional data-dependent pathways that obscure the sensitivity of the
current query. This motivates a privacy-compatible memory principle:
\emph{historical information should enter the private update only through
quantities that have already been privatized}.

Fractional calculus provides a principled framework for modeling
memory-dependent and nonlocal dynamics through power-law historical kernels
\cite{podlubny1999fractional,kilbas2006theory,diethelm2010analysis,
meerschaert2012stochastic,west2014colloquium}. Unlike short-horizon
exponential moving averages, fractional-order memory kernels can represent
long-range temporal dependence, making them a natural mechanism for aggregating
optimization signals across multiple steps. In a DP setting, however, the
historical kernel should not be applied to raw gradients. Instead, the memory
source must be the \emph{private release history}: the sequence of previously
privatized noisy releases. This preserves the intended historical structure
while keeping the memory branch compatible with privacy accounting.

A second challenge is that memory should not be blind or uniform across layers.
A fixed memory horizon treats all layers as equally reliable, even though
different layers may exhibit different training states, spectral organization,
stability, noise sensitivity, and generalization behavior. In non-private deep
learning, the spectra of weight matrices have been widely studied as indicators
of optimization geometry, sharpness, capacity, and generalization
\cite{bartlett2017spectrally,neyshabur2017exploring,keskar2017large,
pennington2017nonlinear,ghorbani2019investigation,foret2021sharpness}. These
observations suggest that layer-wise spectral structure can provide information
about how aggressively each layer should retain or forget historical signals.

We therefore draw on Heavy-Tailed Self-Regularization (HT-SR) and
WeightWatcher-style spectral diagnostics. This line of work analyzes the
empirical spectral density of trained weight matrices and fits power-law tails
to layer spectra~\cite{martin2019traditional,martin2021implicit}. Martin,
Peng, and Mahoney showed that such power-law spectral metrics can diagnose
quality trends of pretrained neural networks, including settings where training
and testing data are unavailable~\cite{martin2021predicting}. In this paper,
we use the resulting spectral exponent not as a privacy mechanism, and not as a
universal certificate of layer quality, but as a layer-wise reliability
heuristic for controlling how quickly private memory should decay. This use is
motivated by spectral diagnostics and evaluated empirically through ablations.

Despite progress in DP optimization, fractional memory modeling, and spectral
diagnostics, their intersection remains underexplored. Existing DP optimizers
primarily focus on clipping, noise calibration, privacy accounting, adaptive
moments, preconditioning, and training stabilization. Fractional methods provide
memory models, but are generally not designed around DP-SGD conditional
sensitivity. WeightWatcher and HT-SR provide layer-wise spectral diagnostics,
but are typically used for model diagnosis rather than as control signals for
privacy-compatible memory. To our knowledge, there has not been a systematic
study of a DP-SGD extension that simultaneously combines private
release-history memory, fractional power-law weighting, layer-wise spectral
reliability, clean conditional sensitivity analysis, and exact reduction to
DP-SGD.

We propose \textbf{SMA-DP-SGD}, a \textbf{Spectral Memory-Aware
Differentially Private Stochastic Gradient Descent} method. At each private
step, SMA-DP-SGD constructs a fractional memory branch for each parameter group
from the private release history, i.e., previously privatized noisy gradient
releases. The formulation is group-wise and general; in our experiments, each
trainable layer is instantiated as one group, so the mechanism operates
layer-wise. A WeightWatcher-inspired power-law exponent provides a
group-specific reliability signal that controls spectral tempering: groups with
favorable spectral organization retain longer effective memory, whereas groups
outside the chosen reliability range forget older releases more aggressively.
The memory contribution is further stabilized using private-history alignment,
norm matching, and warm-up activation. A fixed mixing coefficient controls the
interpolation between the current clipped update and the memory branch; when
this coefficient disables the memory contribution, SMA-DP-SGD exactly reduces
to standard group-wise DP-SGD.

The privacy argument is intentionally transparent. The memory branch is computed
only from the prior private release history, and is therefore fixed after
conditioning on that history. Consequently, the only newly data-dependent
component of the current query is the scaled clipped gradient contribution. This
yields a clean group-wise conditional sensitivity structure while avoiding the
reuse of raw historical gradients. The formal full-step privacy guarantee uses a
conservative joint accountant across parameter groups, while marginal
noise-to-sensitivity ratios are used only for group-wise diagnostic
interpretation.

Our contributions are summarized as follows:
\begin{itemize}
    \item We introduce SMA-DP-SGD, a DP-SGD extension that injects fractional
    memory using previous noisy private releases rather than raw historical
    gradients.
    \item We use WeightWatcher-inspired spectral exponents as layer-wise
    reliability signals to adapt memory decay and effective memory depth.
    \item We provide a clean conditional sensitivity structure and show that
    SMA-DP-SGD exactly recovers standard group-wise DP-SGD when the memory
    contribution is disabled.
    \item We empirically evaluate the method against DP optimization baselines
    and through ablations over spectral intervals, effective memory depth,
    memory ratios, fixed mixing coefficient, shared fractional order, shared
    memory window, and shuffled spectral assignments.
\end{itemize}

The remainder of this paper is organized as follows.
Section~\ref{sec:related_work} reviews related work on differentially private
optimization, fractional memory modeling, and spectral diagnostics.
Section~\ref{sec:methodology} presents SMA-DP-SGD, including the
private-release-history memory construction, spectral tempering, and
group-wise update rule. Section~\ref{sec:experiment} describes the experimental
setup and reports the main empirical results, ablation studies, spectral
diagnostics, and runtime comparison. Section~\ref{sec:broader_impact}
discusses broader impacts. Section~\ref{sec:conclusion} concludes the paper
and summarizes limitations and future directions. Additional experimental
results and diagnostics are provided in Appendix~\ref{app:additional_experiments},
and the full theoretical analysis is provided in
Appendix~\ref{app:sma_dp_sgd_theory}.

\section{Related Work}
\label{sec:related_work}

\paragraph{Differentially private deep learning.}
Differential privacy provides a formal framework for limiting the influence of
individual training examples on the output of a learning algorithm
\cite{dwork2006calibrating,dwork2014algorithmic}. DP-SGD has become the
standard optimization mechanism for private deep learning by combining
per-example gradient clipping with calibrated Gaussian perturbations
\cite{abadi2016deep}. Subsequent work has improved the practicality of private
training through refined privacy accounting, subsampled RDP analysis, Gaussian
DP, adaptive clipping, transfer learning, and large-scale private training
strategies
\cite{mironov2017renyi,wang2019subsampled,bu2020deep,andrew2021differentially,
tramer2021differentially,de2022unlocking}. These methods primarily address the
privacy--utility trade-off through accounting, clipping, noise calibration,
scaling, or optimizer design. In contrast, SMA-DP-SGD focuses on a complementary
question: how to inject useful historical information into private optimization
without reusing raw historical gradients or obscuring the sensitivity of the
current query.

\paragraph{Memory and adaptive optimization.}
Historical information is central to stochastic optimization. Classical
momentum and accelerated methods exploit past directions to stabilize and speed
up optimization, while adaptive methods such as AdaGrad and Adam use accumulated
gradient statistics to rescale updates
\cite{robbins1951stochastic,polyak1964some,nesterov1983method,
duchi2011adaptive,kingma2015adam}. Private optimizers can also contain
optimizer-state memory, but under differential privacy this memory must be
handled carefully because raw gradients and current-batch-dependent gates can
introduce additional data-dependent pathways. SMA-DP-SGD differs from standard
momentum or Adam-style memory by constructing its memory branch only from
previously privatized noisy releases. This private-release-history construction
allows historical information to influence optimization while preserving a
clean conditional sensitivity structure.

\paragraph{Fractional memory modeling.}
Fractional calculus provides a principled framework for modeling nonlocal and
memory-dependent dynamics through power-law historical kernels
\cite{podlubny1999fractional,kilbas2006theory,diethelm2010analysis,
meerschaert2012stochastic,west2014colloquium}. Compared with short-horizon
exponential moving averages, fractional kernels offer a natural mechanism for
representing longer-range temporal dependence. SMA-DP-SGD uses this perspective
to define a tempered fractional memory state over previous private releases.
The key distinction in our setting is privacy compatibility: the fractional
kernel is applied only to already privatized releases, rather than to raw
historical gradients.

\paragraph{Spectral diagnostics and Heavy-Tailed Self-Regularization.}
The spectra of neural-network weight matrices have been studied in connection
with capacity, optimization geometry, sharpness, and generalization
\cite{bartlett2017spectrally,neyshabur2017exploring,keskar2017large,
pennington2017nonlinear,ghorbani2019investigation,foret2021sharpness}.
Heavy-Tailed Self-Regularization and WeightWatcher-style diagnostics analyze
the empirical spectral density of trained weight matrices and fit power-law
tails to their eigenvalue distributions
\cite{martin2019traditional,martin2021implicit}. Martin, Peng, and Mahoney
showed that power-law spectral metrics can diagnose quality trends of
pretrained neural networks even without access to training or testing data
\cite{martin2021predicting}. More recent work has also used
WeightWatcher-style spectral diagnostics as control or stability signals in
privacy-preserving training and machine unlearning: spectral indicators have
been used to adapt gradient-clipping thresholds in differentially private
training, and layer-wise heavy-tailed spectral diagnostics have been used to
reweight unlearning updates according to statistical roughness
\cite{partohaghighi2026clipping,partohaghighi2026roughness}. Our work uses
these spectral diagnostics in a different role: the power-law exponent is
treated as a layer-wise reliability heuristic for controlling the decay of
private fractional memory. We do not use the spectral exponent as a privacy
mechanism or as a universal certificate of layer quality; instead, it modulates
how aggressively each layer forgets older private releases.

\paragraph{Positioning of SMA-DP-SGD.}
SMA-DP-SGD combines three ideas that have largely been studied separately:
differentially private optimization, fractional memory modeling, and spectral
diagnostics. Existing DP methods typically focus on clipping, noise, privacy
accounting, adaptive optimizers, or scaling strategies; fractional methods
provide memory kernels but are not designed around DP-SGD sensitivity; and
WeightWatcher-style diagnostics are usually used for model analysis rather than
as control signals inside private optimization. SMA-DP-SGD bridges these
directions by building a fractional memory branch from the private release
history and adapting its decay using layer-wise spectral reliability, while
retaining an exact reduction to standard group-wise DP-SGD when \(\beta=1\).
\section{Methodology}
\label{sec:methodology}

We propose \textbf{SMA-DP-SGD}, a \textbf{Spectral Memory-Aware
Differentially Private Stochastic Gradient Descent} method. SMA-DP-SGD extends
DP-SGD with a fractional memory branch computed for each predefined parameter
group from the \emph{private release history}, i.e., the sequence of previously
privatized noisy gradient releases. The memory branch is regulated by spectral
tempering, private-history alignment, warm-up activation, and norm matching.

We formulate SMA-DP-SGD for generic parameter groups and instantiate each
trainable layer as one group in our experiments, allowing the spectral memory
mechanism to operate layer-wise. This distinction is important: the privacy
analysis is group-wise and applies to any predefined partition of the trainable
parameters, while the experimental implementation uses the natural layer-wise
partition induced by convolutional and linear layers.

The key design principle is that, at private step \(t\), the memory branch is
constructed only from previously privatized releases and public hyperparameters.
Consequently, conditioned on the prior private release history, the memory
contribution is fixed, and the only newly data-dependent component of the query
is the current clipped subsampled sum scaled by a fixed coefficient \(\beta\).
This structure preserves an exact reduction to standard DP-SGD when
\(\beta=1\), while yielding a clean conditional sensitivity analysis.

\subsection{Problem setup}
\label{subsec:problem_setup}

Let \(D=\{x_i\}_{i=1}^N\) denote the training dataset and let
\(\ell(\theta;x_i)\) be the loss of a model with trainable parameter vector
\(\theta\). We partition the trainable parameters into \(G\) predefined groups,
\(\theta=(\theta^{(1)},\theta^{(2)},\ldots,\theta^{(G)})\), where
\(g\in\{1,\ldots,G\}\) indexes a parameter group.

This group-wise formulation is general: a group may correspond to a layer, a
block, a module, or any predefined subset of trainable parameters. In our
experiments, we use the layer-wise instantiation, where each trainable
convolutional or linear layer is treated as one group. Therefore, all
group-wise quantities, including \(s_t^{(g)}\), \(C^{(g)}\),
\(\rho_t^{(g)}\), \(\lambda_t^{(g)}\), \(\nu_{t-1}^{(g)}\),
\(\Gamma_t^{(g)}\), and \(\Psi_t^{(g)}\), become layer-wise quantities in the
experimental implementation.

At private step \(t\), a Poisson subsample \(S_t\subseteq D\) is drawn with
sampling probability \(q\), where \(q\) controls the expected lot size
\(L=qN\). For each sampled example \(x_i\in S_t\), the group-wise per-example
gradient is \(g_t^{(g)}(x_i)=\nabla_{\theta^{(g)}}\ell(\theta_t;x_i)\). We
apply group-wise clipping,
\[
\bar g_t^{(g)}(x_i)
=
\frac{g_t^{(g)}(x_i)}
{\max\left(1,\lVert g_t^{(g)}(x_i)\rVert_2/C^{(g)}\right)},
\qquad
\lVert \bar g_t^{(g)}(x_i)\rVert_2\le C^{(g)},
\]
where \(C^{(g)}\) controls the group-wise sensitivity of the clipped
contribution. We define the clipped subsampled sum
\(s_t^{(g)}=\sum_{x_i\in S_t}\bar g_t^{(g)}(x_i)\). For each group \(g\), SMA-DP-SGD maintains the private release history
\(\mathcal H_t^{(g)}=\{\tilde s_0^{(g)},\tilde s_1^{(g)},\ldots,
\tilde s_{t-1}^{(g)}\}\), where \(\tilde s_r^{(g)}\) is the noisy private
release produced at step \(r<t\). The history is empty at \(t=0\).

All memory-dependent quantities at step \(t\) are computed from
\(\mathcal H_t^{(g)}\), the model state induced by previous private updates,
and public hyperparameters. The method does not use raw historical gradients,
current unclipped gradients, or the current clipped sum to construct the memory
branch. Hence, after conditioning on \(\mathcal H_t^{(g)}\), the memory
contribution is fixed. The spectral reliability signal is motivated by Heavy-Tailed
Self-Regularization and WeightWatcher-style diagnostics, which analyze the
empirical spectral density of layer weight matrices and fit power-law tails to
their eigenvalue distributions. Prior work has shown that such power-law
spectral metrics can diagnose quality trends in pretrained neural networks,
including settings where training or testing data are unavailable
\cite{martin2021predicting}. In SMA-DP-SGD, these diagnostics are not used as
privacy mechanisms; rather, they provide group-wise reliability signals for
controlling the decay of private memory. Under the layer-wise instantiation used
in our experiments, these reliability signals are computed per trainable layer.

For spectral diagnostics, a group is assumed to contain a matrix-shaped
trainable weight tensor, or a tensor that can be reshaped into a two-dimensional
matrix, as in convolutional and linear layers. Following WeightWatcher-style
spectral diagnostics~\cite{martin2021predicting}, for a matrix-shaped group
with weight matrix \(W_t^{(g)}\), we form
\(X_t^{(g)}=(W_t^{(g)})^\top W_t^{(g)}\). For convolutional kernels, the kernel
tensor is first reshaped into an appropriate two-dimensional matrix before
forming the corresponding spectral matrix. Let \(\{\lambda_i^{(g)}\}\) denote
the eigenvalues of \(X_t^{(g)}\). We estimate a WeightWatcher-style power-law
exponent \(\rho_t^{(g)}\), which serves as a group-wise spectral reliability
signal, by fitting the upper tail of the empirical eigenvalue distribution,
\(p(\lambda)\propto\lambda^{-\rho_t^{(g)}}\). We use the notation
\(\rho_t^{(g)}\), rather than the conventional power-law symbol \(\alpha\), to
avoid conflict with the fractional memory order \(\alpha\).

The exponent \(\rho_t^{(g)}\) is computed from the current model state
\(\theta_t\), which is induced by previous private releases. Thus, conditioned
on the prior private release history and public hyperparameters,
\(\rho_t^{(g)}\) is fixed for the current query. Consequently, the spectral
tempering mechanism affects only the already-determined memory branch and does
not introduce additional current-batch dependence.

Motivated by the HT-SR/WeightWatcher interpretation, we treat intermediate
power-law exponent ranges as favorable spectral-reliability regions. Let
\(I_\rho=[\rho_{\min},\rho_{\max}]\) denote the spectral reliability interval
used to decide whether a group should retain longer memory or forget older
releases more aggressively. This interval is a tunable design choice rather
than a universal criterion for layer quality. To measure how far the current
spectral exponent lies outside the reliability interval, define
\[
d_t^{(g)}(I_\rho)
=
\max\left(
0,
\rho_{\min}-\rho_t^{(g)},
\rho_t^{(g)}-\rho_{\max}
\right).
\]
Equivalently, with
\(m_\rho=(\rho_{\min}+\rho_{\max})/2\) and
\(h_\rho=(\rho_{\max}-\rho_{\min})/2\), the same deviation can be written
compactly as
\[
d_t^{(g)}(I_\rho)
=
\max\left(
0,
\big |\rho_t^{(g)}-m_\rho\big |-h_\rho
\right).
\]
Thus, \(d_t^{(g)}(I_\rho)=0\) whenever \(\rho_t^{(g)}\in I_\rho\), and otherwise
measures the distance of \(\rho_t^{(g)}\) from the closest endpoint of the
reliability interval. In the default configuration, we use
\(I_\rho^\star=[2,6]\), corresponding to \(m_\rho=4\) and \(h_\rho=2\). We do
not assume that this interval is universally optimal; its role is to define the
reliability range used by the spectral tempering mechanism.

The spectral deviation controls the group-wise tempering coefficient
\[
\lambda_t^{(g)}(I_\rho)
=
1-\exp\left(
-c_\lambda d_t^{(g)}(I_\rho)
\right),
\qquad c_\lambda>0,
\]
where \(c_\lambda\) controls how strongly spectral deviation is converted into
memory tempering. When \(\rho_t^{(g)}\in I_\rho\), the deviation is zero and
\(\lambda_t^{(g)}(I_\rho)=0\), so the corresponding group retains longer
effective memory. When \(\rho_t^{(g)}\notin I_\rho\), the deviation is positive
and \(\lambda_t^{(g)}(I_\rho)>0\), so older private releases are down-weighted
more aggressively. In this work, SMA-DP-SGD uses a shared memory window \(K\), which sets the
maximum number of previous private releases used by the memory branch, and a
shared fractional order \(\alpha\in(0,1]\), which controls the power-law decay
of the fractional memory kernel, across all parameter groups. Group-wise
adaptivity enters through the spectral tempering coefficient
\(\lambda_t^{(g)}(I_\rho)\), which depends on the group-wise spectral exponent
\(\rho_t^{(g)}\). At step \(t\), the number of available memory terms is
\(M_t=\min(K-1,t)\). For \(j=1,\ldots,M_t\), define
\[
a_{t,j}^{(g)}
=
(j+1)^{\alpha-1}
\exp\left(-\lambda_t^{(g)}(I_\rho)j\right)
\quad
\text{and}
\quad 
\hat a_{t,j}^{(g)}
=
\frac{a_{t,j}^{(g)}}
{\sum_{\ell=1}^{M_t}a_{t,\ell}^{(g)}}.
\]
The private fractional memory state,
$\nu_{t-1}^{(g)},$ and
effective memory depth,
$D_{\mathrm{eff},t}^{(g)},$
are given by 
\[
\nu_{t-1}^{(g)}
=
\sum_{j=1}^{M_t}
\hat a_{t,j}^{(g)}
\tilde s_{t-j}^{(g)}
\quad \text{and} \quad 
D_{\mathrm{eff},t}^{(g)}
=
\sum_{j=1}^{M_t}j\hat a_{t,j}^{(g)},
\]
respectively. 
If \(M_t=0\), then \(\nu_{t-1}^{(g)}=0\). The effective memory depth
\(D_{\mathrm{eff},t}^{(g)}\) summarizes how far into the private release history
the group effectively looks. Although \(K\) and \(\alpha\) are shared, the
normalized weights \(\hat a_{t,j}^{(g)}\) and the resulting memory depth remain
group-wise because they depend on the group-wise tempering coefficient
\(\lambda_t^{(g)}(I_\rho)\). Smaller tempering values generally yield longer
effective memory, whereas larger values shorten the memory horizon.
Historical memory can be harmful when it points in a direction inconsistent
with the recent private optimization trend. To suppress stale or contradictory
memory, SMA-DP-SGD constructs a private exponential moving average trend from
previous releases:
\[
\mu_{t-1}^{(g)}
=
\gamma_{\text{ema}}\tilde s_{t-1}^{(g)}
+
(1-\gamma_{\text{ema}})\mu_{t-2}^{(g)},
\qquad
\gamma_{\text{ema}}\in(0,1],
\]
where \(\gamma_{\text{ema}}\) controls the recency of the private EMA trend,
with larger values emphasizing more recent private releases. We initialize
\(\mu_0^{(g)}=\tilde s_0^{(g)}\). At \(t=0\), no previous release exists, and
the memory branch is set to zero.

The private-history alignment gate is
\[
\Gamma_t^{(g)}
=
\max\left(
0,
\frac{
\langle \mu_{t-1}^{(g)},\nu_{t-1}^{(g)}\rangle
}{
\lVert\mu_{t-1}^{(g)}\rVert_2
\lVert\nu_{t-1}^{(g)}\rVert_2+\epsilon
}
\right),
\]
where \(\epsilon>0\) is a numerical stability constant that prevents division
by very small norms. The gate suppresses memory directions that are not aligned
with the recent private trend. Even aligned memory may have an inappropriate
magnitude, so we use the norm-matching scale
\[
\Psi_t^{(g)}
=
\min\left(
\xi_{\max},
\frac{
\lVert\mu_{t-1}^{(g)}\rVert_2
}{
\lVert\nu_{t-1}^{(g)}\rVert_2+\epsilon
}
\right),
\qquad \xi_{\max}>0,
\]
where \(\xi_{\max}\) caps memory amplification induced by norm matching. Both
\(\Gamma_t^{(g)}\) and \(\Psi_t^{(g)}\) are computed only from previous private
releases. Thus, the memory contribution is directionally filtered and
scale-controlled before entering the private query.
\begin{figure}[t]
\centering
\includegraphics[height=8.5cm
]{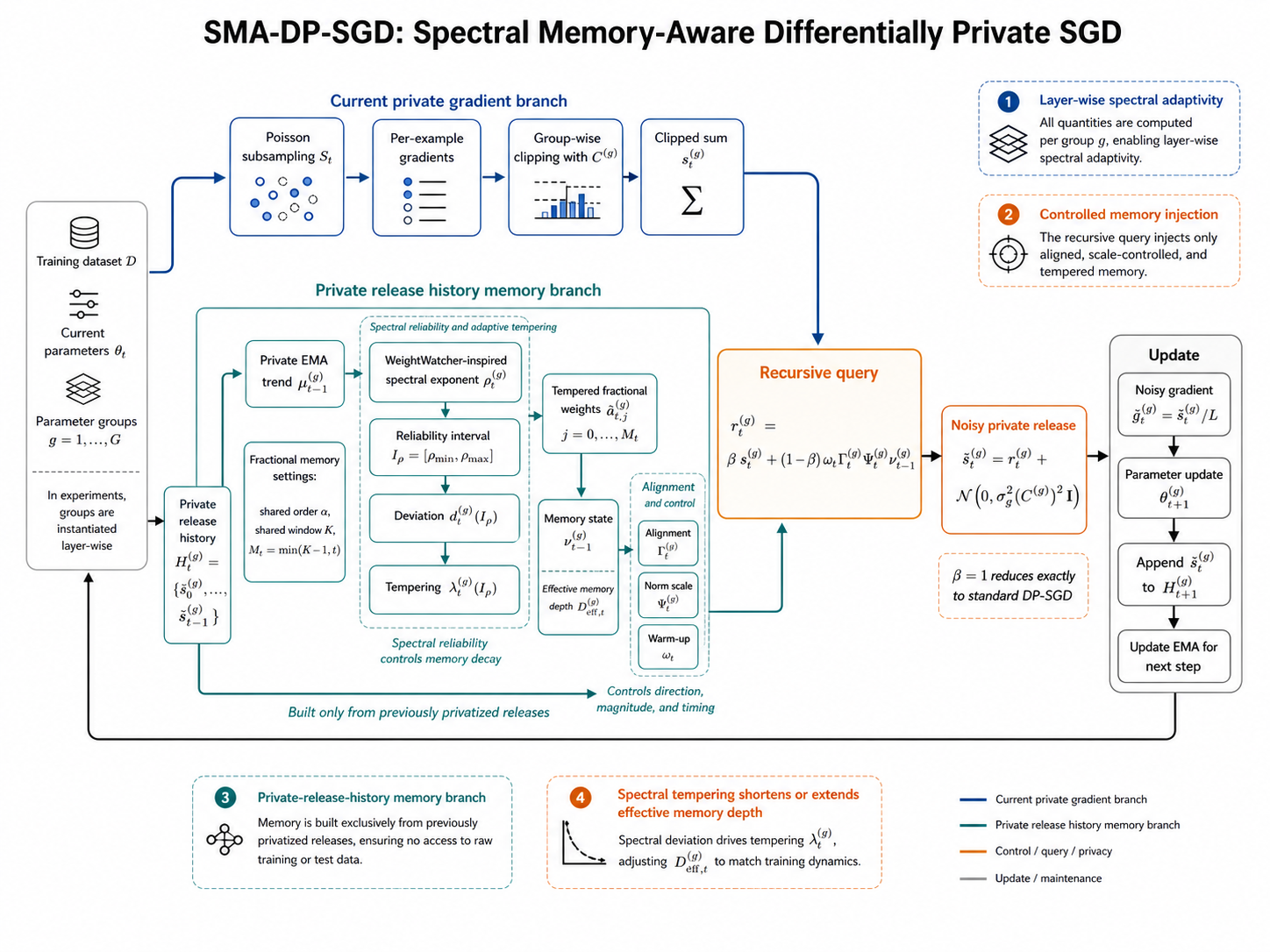}
    \caption{Overview of SMA-DP-SGD, combining a current private gradient branch with a spectrally tempered private-release-history memory branch.}
\label{fig:sma_dp_sgd_overview}
\end{figure}
Algorithm~\ref{alg:sma_dp_sgd} summarizes SMA-DP-SGD. The method uses a
deterministic warm-up coefficient
\(\omega_t=1-\exp(-t/\tau_{\mathrm{warm}})\), where
\(\tau_{\mathrm{warm}}>0\) controls how quickly the memory branch becomes
active, and a fixed current-memory mixing coefficient \(\beta\in(0,1]\), which
controls the interpolation between the current clipped update and the
private-release-history memory branch. The recursive query for group \(g\) is
\[
r_t^{(g)}
=
\beta s_t^{(g)}
+
(1-\beta)\omega_t
\Gamma_t^{(g)}
\Psi_t^{(g)}
\nu_{t-1}^{(g)}.
\]
Equivalently,
\(r_t^{(g)}=\beta s_t^{(g)}+b_t^{(g)}(\mathcal H_t^{(g)})\), where
\(b_t^{(g)}(\mathcal H_t^{(g)})\) is the memory branch computed from the prior
private release history. The case \(\beta=1\) removes the memory branch and
recovers standard group-wise DP-SGD.

The private release is
\(\tilde s_t^{(g)}=r_t^{(g)}+Z_t^{(g)}\), with
\(Z_t^{(g)}\sim\mathcal N(0,\sigma_g^2(C^{(g)})^2I)\), where \(\sigma_g\)
controls the Gaussian noise strength for group \(g\). The noisy gradient
estimate is \(\tilde g_t^{(g)}=\tilde s_t^{(g)}/L\), and the group-wise
parameter update is
\(\theta_{t+1}^{(g)}=\theta_t^{(g)}-\eta \tilde g_t^{(g)}\). After the update,
\(\tilde s_t^{(g)}\) is appended to \(\mathcal H_{t+1}^{(g)}\), and the private
EMA trend is updated for the next step.
\begin{algorithm}[t]
\caption{SMA-DP-SGD with Private Release-History Memory}
\label{alg:sma_dp_sgd}
\tiny
\setlength{\tabcolsep}{2pt}
\begin{algorithmic}[1]
\Require Dataset \(D\), loss \(\ell(\theta;x)\)
\Require Predefined parameter groups
\(\theta=(\theta^{(1)},\ldots,\theta^{(G)})\), instantiated as trainable layers
in the experiments
\Require Sampling rate \(q\), learning rate \(\eta\), clipping norms \(C^{(g)}\),
noise multipliers \(\sigma_g\)
\Require Shared memory window \(K\), shared fractional order \(\alpha\), fixed
mixing coefficient \(\beta\in(0,1]\)
\Require Spectral interval \(I_\rho=[\rho_{\text{min}},\rho_{\text{max}}]\), tempering
constant \(c_\lambda\)
\Require EMA coefficient \(\gamma_{\text{ema}}\), warm-up scale
\(\tau_{\text{warm}}\), norm cap \(\xi_{\text{max}}\), stability constant
\(\epsilon>0\)
\State Initialize \(\theta_0\), histories \(\mathcal H_0^{(g)}\gets\emptyset\),
and trends \(\mu_{-1}^{(g)}\gets0\) for all \(g\)
\For{\(t=0,\ldots,T-1\)}
    \State Draw Poisson subsample \(S_t\) with probability \(q\)
    \State Set \(\omega_t\gets1-\text{exp}(-t/\tau_{\text{warm}})\) and
    \(M_t\gets\text{min}(K-1,t)\)
    \For{\(g=1,\ldots,G\)}
        \State Compute per-example gradients, clip them at \(C^{(g)}\), and form
        \(s_t^{(g)}=\sum_{x_i\in S_t}\bar g_t^{(g)}(x_i)\)
        \If{\(t=0\)}
            \State Set \(\nu_{t-1}^{(g)}\gets0\), \(\Gamma_t^{(g)}\gets0\),
            and \(\Psi_t^{(g)}\gets0\)
        \Else
            \State Estimate \(\rho_t^{(g)}\) from the matrix-shaped weight
            tensor of group \(g\), when applicable
            \State Compute \(d_t^{(g)}(I_\rho)
            \gets \text{max}(0,\rho_{\text{min}}-\rho_t^{(g)},\rho_t^{(g)}-\rho_{\text{max}})\)
            \State Compute
            \(\lambda_t^{(g)}(I_\rho)\gets
            1-\text{exp}(-c_\lambda d_t^{(g)}(I_\rho))\)
            \State For \(j=1,\ldots,M_t\), compute
            \(a_{t,j}^{(g)}\gets
            (j+1)^{\alpha-1}\text{exp}(-\lambda_t^{(g)}(I_\rho)j)\)
            and
            \(\hat a_{t,j}^{(g)}
            \gets a_{t,j}^{(g)}/\sum_{\ell=1}^{M_t}a_{t,\ell}^{(g)}\)
            \State Construct
            \(\nu_{t-1}^{(g)}
            \gets \sum_{j=1}^{M_t}\hat a_{t,j}^{(g)}
            \tilde s_{t-j}^{(g)}\)
            \State Compute
            \(\Gamma_t^{(g)}
            \gets
            \text{max}\!\left(
            0,
            \frac{
            \langle \mu_{t-1}^{(g)},\nu_{t-1}^{(g)}\rangle
            }{
            \lVert\mu_{t-1}^{(g)}\rVert_2
            \lVert\nu_{t-1}^{(g)}\rVert_2+\epsilon
            }
            \right)\)
            \State Compute
            \(\Psi_t^{(g)}
            \gets
            \text{min}\!\left(
            \xi_{\text{max}},
            \frac{
            \lVert\mu_{t-1}^{(g)}\rVert_2
            }{
            \lVert\nu_{t-1}^{(g)}\rVert_2+\epsilon
            }
            \right)\)
        \EndIf
        \State Form
        \(r_t^{(g)}
        \gets
        \beta s_t^{(g)}
        +(1-\beta)\omega_t\Gamma_t^{(g)}
        \Psi_t^{(g)}\nu_{t-1}^{(g)}\)
        \State Release
        \(\tilde s_t^{(g)}
        \gets
        r_t^{(g)}
        +
        \mathcal N(0,\sigma_g^2(C^{(g)})^2I)\)
        \State Update
        \(\theta_{t+1}^{(g)}
        \gets
        \theta_t^{(g)}
        -
        \eta\tilde s_t^{(g)}/L\)
        \State Append \(\tilde s_t^{(g)}\) to \(\mathcal H_{t+1}^{(g)}\)
        \State Update
        \(\mu_t^{(g)}
        \gets
        \tilde s_0^{(g)}\) if \(t=0\), otherwise
        \(\mu_t^{(g)}
        \gets
        \gamma_{\text{ema}}\tilde s_t^{(g)}
        +(1-\gamma_{\text{ema}})\mu_{t-1}^{(g)}\)
    \EndFor
\EndFor
\State \Return \(\theta_T\)
\end{algorithmic}
\end{algorithm}
Figure~\ref{fig:sma_dp_sgd_overview} summarizes the SMA-DP-SGD workflow, where the current clipped private update is combined with a spectrally tempered memory branch built from previously privatized releases.

\section{Experiment}
\label{sec:experiment}

We evaluate SMA-DP-SGD on CIFAR-10 and CIFAR-100~\cite{krizhevsky2009learning}
and MNIST~\cite{lecun1998gradient} against representative differentially
private optimization baselines, including DP-SGD, DP-Adam, DP-AdamW, DP-IS,
DP-SAM, DP-SAT, and DP-Adam-AC. The main experimental section reports the
cross-dataset utility comparison, while additional ablations, stage-wise
spectral diagnostics, final-accuracy statistics, marginal privacy-cost
diagnostics, spectral-interval sensitivity analyses, and runtime comparisons
are provided in Appendix~\ref{app:additional_experiments}. 

\subsection{Cross-Dataset Accuracy Comparison under Differential Privacy}
\label{subsec:cross-dataset-accuracy}

As shown in Fig.~\ref{fig:cross-dataset-accuracy2}, SMA-DP-SGD exhibits strong
and stable performance across CIFAR-100, CIFAR-10, and MNIST. On CIFAR-100,
which is the most challenging benchmark among the three due to its larger
number of classes and higher inter-class similarity, SMA-DP-SGD achieves the
highest final test accuracy among the compared methods. This suggests that the
proposed spectral memory-aware mechanism can improve learning under noisy
private gradients, especially in more complex classification settings where
standard DP optimizers may suffer from unstable or under-optimized convergence.

On CIFAR-10, most adaptive DP optimizers reach relatively similar accuracy
levels after the early training phase. Nevertheless, SMA-DP-SGD remains highly
competitive throughout training and achieves one of the strongest final accuracy
values. This indicates that the proposed method does not merely benefit from
dataset-specific behavior on CIFAR-100, but also generalizes well to moderately
complex image classification tasks.

For MNIST, several adaptive methods, including DP-Adam, DP-AdamW, DP-SAM, and
DP-SAT, quickly reach high accuracy due to the relative simplicity of the
dataset. SMA-DP-SGD also achieves high final accuracy and remains close to the
best-performing methods. This result indicates that the proposed method
preserves stable learning behavior on simpler datasets while offering stronger
advantages on more challenging benchmarks.

\begin{figure}[t]
\centering
\begin{tabular}{ccc}
\hspace{-.15cm}
\includegraphics[
width=4.5cm]
{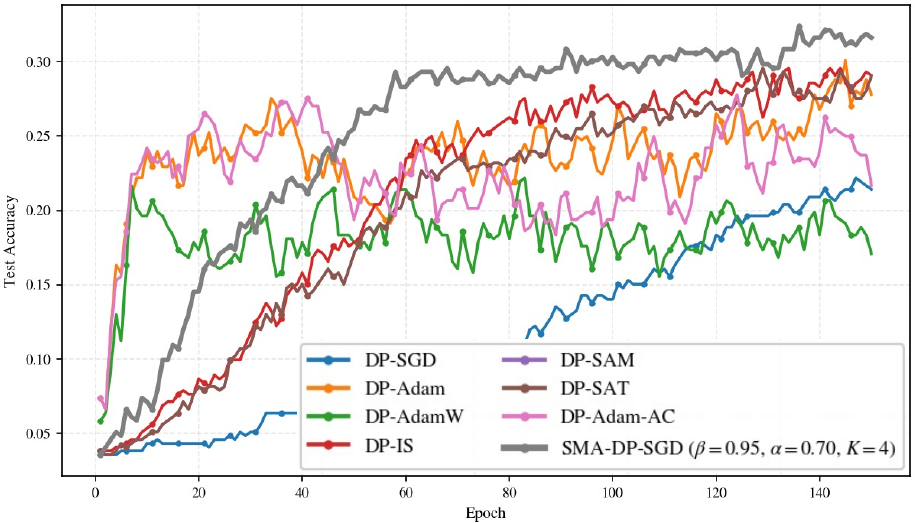}
&
\hspace{-.4cm}
\includegraphics[
width=4.5cm]
{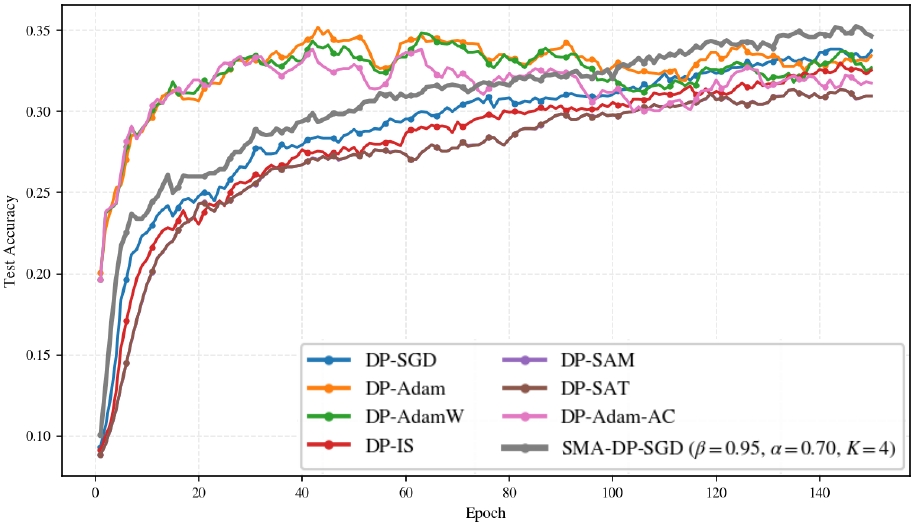}
&
\hspace{-.4cm}
\includegraphics[
width=4.5cm]
{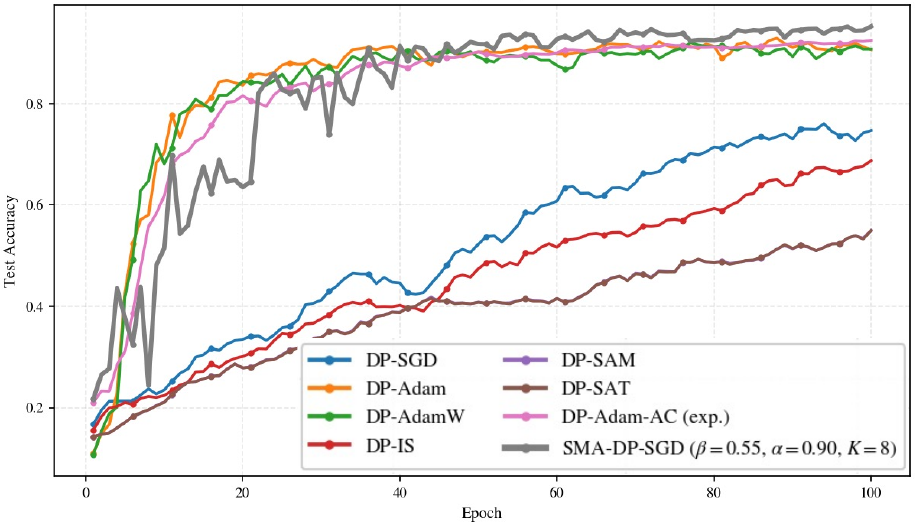}\\
{\small  (a) CIFAR-100} &
{\small (b) CIFAR-10} &
{\small (c) MNIST}
\end{tabular}
\caption{
Cross-dataset comparison of test accuracy versus training epoch for
differentially private optimizers on (a) CIFAR-100, (b) CIFAR-10, and (c) MNIST. SMA-DP-SGD
is compared against representative DP optimization baselines under matched
experimental settings.
}
\label{fig:cross-dataset-accuracy2}
\end{figure}


\section{Conclusion}
\label{sec:conclusion}

This paper introduced SMA-DP-SGD, a spectral memory-aware extension of DP-SGD
that reuses historical information in a privacy-compatible manner. The method
constructs a fractional memory branch only from the private release history,
namely previously privatized noisy gradient releases, rather than from raw
historical gradients. Conditioned on this prior private history, the memory
branch is fixed, so the update preserves a clean sensitivity structure while
retaining an exact reduction to standard group-wise DP-SGD when the memory
contribution is disabled.

SMA-DP-SGD combines fractional memory, WeightWatcher-inspired spectral
tempering, private-history alignment, norm matching, and warm-up activation. In
experiments, the group-wise formulation is instantiated layer-wise, allowing
each trainable convolutional or linear layer to receive a spectral reliability
signal. Across CIFAR-100, CIFAR-10, and MNIST, SMA-DP-SGD achieves competitive
or superior accuracy compared with representative DP optimization baselines,
with particularly strong behavior on more challenging image-classification
settings. Ablations and diagnostics show that the fixed mixing coefficient
controls the privacy--utility behavior, while the spectral-memory mechanism
maintains a controlled effective memory depth and a small memory-branch
contribution.

Several limitations remain. Because the memory branch is built from previous
noisy private releases, it can inherit historical Gaussian noise along with
useful optimization signal. The spectral reliability interval is also a tunable
heuristic rather than a universal criterion for layer quality, and the spectral
diagnostics introduce additional computational overhead. Moreover, full-step
privacy guarantees require conservative joint accounting, while marginal
privacy-cost curves should be interpreted only as diagnostics.

Future work should develop noise-aware private memory mechanisms that reduce
the influence of inherited historical noise while remaining compatible with
differential privacy. Other directions include adaptive spectral reliability
ranges, cheaper spectral diagnostics, and extensions to federated learning,
personalized private optimization, and private fine-tuning of larger models.

\section{Broader Impact}
\label{sec:broader_impact}

SMA-DP-SGD aims to improve the utility of differentially private training by
reusing historical information only through previously privatized releases.
This may benefit privacy-sensitive applications such as healthcare, finance,
mobile learning, and user-behavior modeling, where useful models must be trained
under formal privacy constraints. However, improved private optimization does
not remove the need for careful privacy accounting, transparent reporting of
privacy budgets, and rigorous implementation. In particular, marginal
privacy-cost diagnostics should not be interpreted as full-model privacy
guarantees. The method also introduces additional computational overhead through
spectral diagnostics and memory management. Responsible deployment should
therefore combine SMA-DP-SGD with reproducible evaluation, conservative privacy
analysis, and domain-specific risk assessment.


\bibliographystyle{plainnat}
\bibliography{references}

\newpage

\clearpage
\appendix

\section*{Appendix Overview}
\phantomsection
\label{app:overview}
\addcontentsline{toc}{section}{Appendix Overview}

This appendix provides supplementary experimental details, additional empirical
results, diagnostic analyses, runtime comparisons, and the full theoretical
analysis supporting the main paper. To improve readability and navigation, the
supplementary material is organized into two main parts.
Appendix~\ref{app:additional_experiments} provides additional experiments and
reproducibility details, while Appendix~\ref{app:sma_dp_sgd_theory} presents
the theoretical analysis and interpretation of SMA-DP-SGD.

\paragraph{Appendix~\ref{app:additional_experiments}: Additional Experiments and Reproducibility Details.}
This part complements the main experimental section. It includes the full
experimental setup and reproducibility information, final-accuracy statistics on
CIFAR-10, the diagnostic effect of the fixed mixing parameter on marginal
privacy cost, effective fractional-memory diagnostics, and runtime comparisons.
It also clarifies the layer-wise experimental instantiation of the group-wise
formulation, the use of a shared fractional order \(\alpha\), a shared memory
window \(K\), and the reporting protocol for repeated runs and confidence
intervals.

\paragraph{Appendix~\ref{app:sma_dp_sgd_theory}: Theoretical Analysis and Interpretation of SMA-DP-SGD.}
This part analyzes SMA-DP-SGD from the perspective of recursive-query
sensitivity, conservative joint R\'enyi differential privacy accounting,
adaptive composition, mechanism-level interpretation, signal--memory--noise
decomposition, and limiting regimes. It distinguishes the marginal group-wise
noise-to-sensitivity diagnostic from the formal full-step joint accountant and
establishes the exact reduction to group-wise DP-SGD when the memory
contribution is disabled.

\paragraph{Navigation Guide.}
For convenience, the appendix is organized as follows:
\begin{itemize}
    \item Appendix~\ref{app:additional_experiments}: Additional experiments and reproducibility details.
    \item Appendix~\ref{app:experimental_setup_reproducibility}: Experimental setup, baselines, parameter grouping, and reproducibility protocol.
    \item Appendix~\ref{app:cifar10_final_accuracy_comparison}: Final accuracy comparison on CIFAR-10 with mean, standard deviation, and 95\% confidence intervals.
    \item Appendix~\ref{app:cifar10_beta_ablation_privacy_budget}: Diagnostic effect of the fixed mixing parameter on marginal privacy cost.
    \item Appendix~\ref{app:cifar10_spectral_interval_memory_diagnostics}: Diagnostic analysis of effective fractional memory depth and memory-branch ratio.
    \item Appendix~\ref{app:cifar10_runtime_comparison}: Runtime comparison and computational overhead on CIFAR-10.
    \item Appendix~\ref{app:sma_dp_sgd_theory}: Theoretical analysis and interpretation of SMA-DP-SGD.
    \item Appendix~\ref{subsec:theory_assumptions}: Standing assumptions and notation.
    \item Appendix~\ref{subsec:theory_recursive_sensitivity}: Private-history-conditioned recursive-query sensitivity.
    \item Appendix~\ref{subsec:theory_marginal_vs_joint}: Distinction between marginal group-wise diagnostics and full joint privacy accounting.
    \item Appendix~\ref{subsec:theory_joint_rdp}: Conservative joint per-step RDP accounting.
    \item Appendix~\ref{subsec:theory_adaptive_composition}: Adaptive composition over training.
    \item Appendix~\ref{subsec:theory_signal_memory_noise}: Signal--memory--noise decomposition.
    \item Appendix~\ref{subsec:theory_interpretation}: Interpretation of spectral tempering and private-history gating.
    \item Appendix~\ref{subsec:theory_special_cases}: Special cases and limiting regimes.
\end{itemize}

\clearpage

\clearpage

\appendix

\section{Additional Experiments and Reproducibility Details}
\label{app:additional_experiments}

\subsection{Experimental Setup and Reproducibility}
\label{app:experimental_setup_reproducibility}

\paragraph{Datasets and tasks.}
We evaluate SMA-DP-SGD on three standard image-classification benchmarks:
CIFAR-100, CIFAR-10, and MNIST. These datasets provide different levels of
visual complexity and allow us to assess whether the proposed
private-release-history memory mechanism remains effective across both simple
and more challenging private learning settings.

\paragraph{Baselines.}
We compare SMA-DP-SGD against representative differentially private optimization
baselines, including DP-SGD, DP-Adam, DP-AdamW, DP-IS, DP-SAM, DP-SAT, and
DP-Adam-AC. All methods are evaluated under matched dataset, privacy, model,
and training configurations whenever they appear in the same comparison.

\paragraph{Model and parameter grouping.}
The methodology is formulated for generic parameter groups \(g\). In the
experiments, we instantiate this group-wise formulation layer-wise: each
trainable convolutional or linear layer is treated as one parameter group.
Consequently, group-wise quantities such as \(C^{(g)}\), \(s_t^{(g)}\),
\(\rho_t^{(g)}\), \(\lambda_t^{(g)}(I_\rho)\), \(\nu_{t-1}^{(g)}\),
\(\Gamma_t^{(g)}\), and \(\Psi_t^{(g)}\) are computed per layer. For the
CIFAR-10 spectral-dynamics study, we use a ResNet-18 model and aggregate
layer-wise spectral quantities into stage-wise averages for visualization.

\paragraph{SMA-DP-SGD configuration.}
Unless otherwise stated, SMA-DP-SGD uses a fixed current-memory mixing
coefficient \(\beta\), a shared fractional order \(\alpha\), and a shared memory
window \(K\) across all parameter groups. The memory branch is constructed only
from the private release history, i.e., previously privatized noisy releases,
and is controlled by spectral tempering, private-history alignment, norm
matching, and warm-up activation. The case \(\beta=1\) corresponds to the exact
DP-SGD limiting case with no memory contribution.

\paragraph{Repeated runs and statistical reporting.}
For final-accuracy comparisons on CIFAR-10, each method is evaluated over three
independent runs. We report the mean final test accuracy, sample standard
deviation, and two-sided \(95\%\) confidence interval computed using the
Student-\(t\) interval.

\paragraph{Reproducibility and code availability.}
To support reproducibility, all compared methods use the same dataset split,
model configuration, training horizon, sampling protocol, clipping setup, and
privacy noise configuration within each experiment. The main experimental
configuration is summarized in Table~\ref{tab:appendix_experimental_config}.
The source code is not released at submission time. If the paper is published,
we plan to provide the implementation and scripts upon reasonable request to
support reproduction of the reported results.

\begin{table}[t]
\centering
\caption{
Summary of the main experimental configuration. Hyperparameters that vary
across ablations are reported in the corresponding experiment-specific rows.
}
\label{tab:appendix_experimental_config}
\scriptsize
\setlength{\tabcolsep}{4pt}
\resizebox{\linewidth}{!}{%
\begin{tabular}{ll}
\toprule
Quantity & Value / Description \\
\midrule
Datasets & CIFAR-100, CIFAR-10, MNIST \\
Baselines & DP-SGD, DP-Adam, DP-AdamW, DP-IS, DP-SAM, DP-SAT, DP-Adam-AC \\
Parameter grouping & Layer-wise; each trainable convolutional or linear layer is one group \\
Spectral-dynamics model & ResNet-18 on CIFAR-10 \\
Spectral interval & \(I_\rho=[2,6]\), unless otherwise stated \\
Mixing coefficient & Fixed \(\beta\); ablation uses \(\{1.00,0.90,0.70,0.50\}\) \\
Final-accuracy setting & \(\beta=0.95\), \(\alpha=0.70\), \(K=4\) \\
Spectral-interval diagnostic & Candidate intervals \([1,3]\), \([2,4]\), \([2,6]\), \([3,5]\), \([4,6]\), \([5,7]\) \\
Runtime setting & \(\beta=0.55\), \(\alpha=0.90\), \(K=8\) \\
Repeated runs & Three independent runs for final-accuracy statistics \\
Statistical reporting & Mean, sample standard deviation, and two-sided \(95\%\) Student-\(t\) CI \\
Privacy-cost curves & Marginal group-wise diagnostic using \(\sigma/\beta\) \\
Formal privacy guarantee & Conservative joint accountant in Appendix~\ref{app:sma_dp_sgd_theory} \\
Code availability & Not released at submission time; available upon reasonable request after publication \\
\bottomrule
\end{tabular}%
}
\end{table}

\subsection{Ablation Study on the Fixed Mixing Parameter on CIFAR-10}
\label{app:beta-ablation-cifar10}

To analyze the role of the fixed current-memory mixing parameter \(\beta\), we
evaluate the privacy--utility behavior of SMA-DP-SGD on CIFAR-10 under
different fixed values of \(\beta\). We compare
\(\beta \in \{1.00,0.90,0.70,0.50\}\), where \(\beta=1.00\) corresponds to the
standard DP-SGD limiting case, while smaller values introduce stronger
private-release-history fractional memory effects into the optimization
dynamics. The resulting test accuracy as a function of accumulated privacy
budget is shown in Fig.~\ref{fig:beta-ablation-privacy-accuracy-cifar10}.

\begin{figure}[t]
\centering
\includegraphics[width=0.72\linewidth]{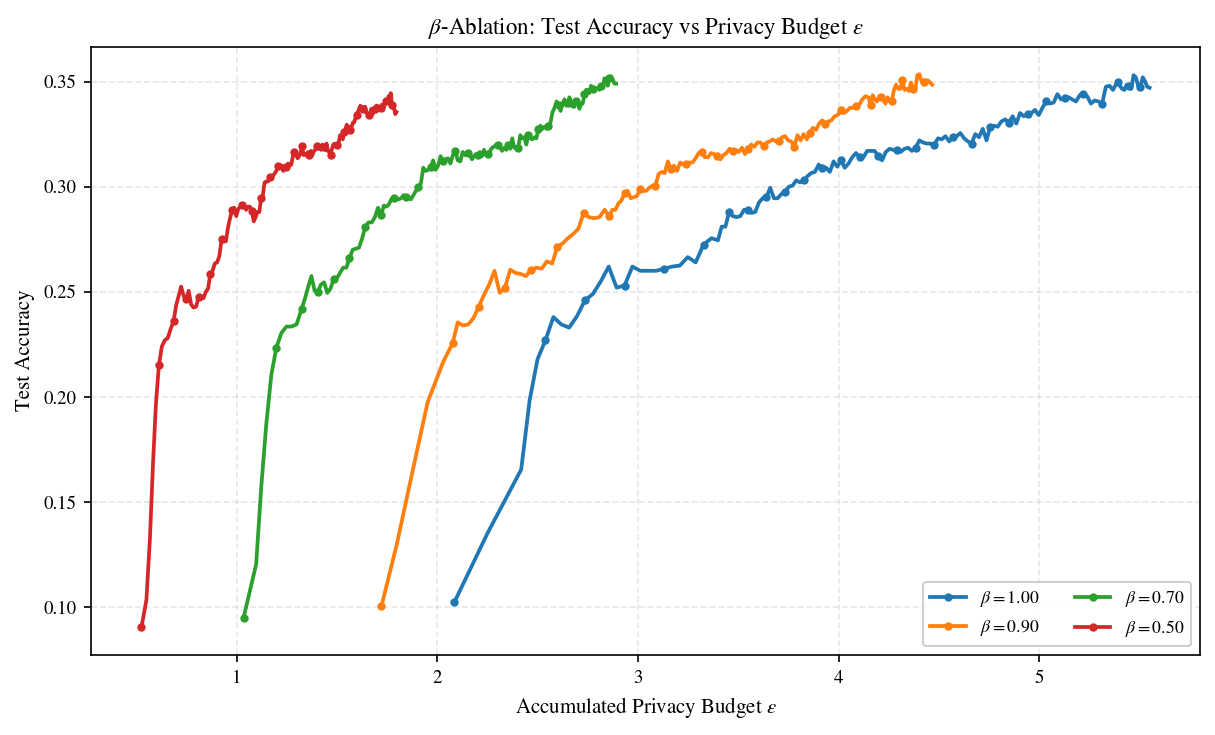}
\caption{
Ablation study of the fixed current-memory mixing parameter \(\beta\) for
SMA-DP-SGD on CIFAR-10. The figure reports test accuracy versus accumulated
privacy budget for \(\beta \in \{1.00,0.90,0.70,0.50\}\).
}
\label{fig:beta-ablation-privacy-accuracy-cifar10}
\end{figure}

As shown in Fig.~\ref{fig:beta-ablation-privacy-accuracy-cifar10}, the choice
of \(\beta\) substantially affects the privacy--utility trajectory. The limiting
case \(\beta=1.00\) requires a larger privacy budget to reach high test
accuracy. In contrast, smaller values of \(\beta\), especially \(\beta=0.70\)
and \(\beta=0.50\), reach comparable or higher accuracy using a smaller privacy
budget in this diagnostic. The curve for \(\beta=0.90\) provides an
intermediate trade-off between the standard DP-SGD update and stronger
memory-based regimes. These results suggest that private-release-history
fractional memory can improve optimization efficiency when the memory
contribution remains controlled.

\subsection{Stage-Wise Spectral Dynamics and Tempering on CIFAR-10}
\label{app:stage_wise_spectral_dynamics}

To further analyze SMA-DP-SGD, we study the evolution of the
WeightWatcher-style spectral exponent and its induced tempering coefficient
across different stages of the CIFAR-10 ResNet-18 model. Since ResNet-18
contains many convolutional and linear parameter groups, directly plotting all
group-wise spectral trajectories would be visually cluttered. Therefore, we
aggregate layer-wise quantities into stage-wise averages.

Let \(\mathcal{G}_{\mathrm{stage}}\) denote the set of convolutional or linear
parameter groups belonging to a given ResNet stage, such as the stem, stage1,
stage2, stage3, stage4, or the final classifier. For each epoch \(t\), we define
the stage-wise spectral exponent as
\begin{equation}
\bar{\rho}_{t}^{(\mathrm{stage})}
=
\frac{1}{|\mathcal{G}_{\mathrm{stage}}|}
\sum_{g \in \mathcal{G}_{\mathrm{stage}}}
\rho_{t}^{(g)}.
\label{eq:stage_avg_rho_experiment}
\end{equation}
Similarly, we define the stage-wise tempering coefficient as
\begin{equation}
\bar{\lambda}_{t}^{(\mathrm{stage})}
=
\frac{1}{|\mathcal{G}_{\mathrm{stage}}|}
\sum_{g \in \mathcal{G}_{\mathrm{stage}}}
\lambda_{t}^{(g)}(I_\rho).
\label{eq:stage_avg_lambda_experiment}
\end{equation}
These diagnostics summarize the spectral evolution and induced memory-decay
policy of each major ResNet component while preserving the layer-aware nature of
SMA-DP-SGD.

The favorable spectral-reliability interval \(I_\rho=[2,6]\) is used as a
heuristic for controlling the memory horizon. Groups closer to this interval
receive weaker tempering and retain longer private-release-history memory,
whereas groups farther from the interval receive stronger tempering and forget
older releases more aggressively.

Figure~\ref{fig:stage_wise_spectral_and_tempering_cifar10} reports both the
stage-wise spectral exponent trajectories and the induced stage-wise tempering
coefficients. The results show that different ResNet components occupy
different spectral regimes during private optimization. The tempering plot
illustrates how these spectral states are translated into stage-dependent
memory-decay coefficients. This supports the use of a layer-wise spectral
memory policy rather than a single uniform memory horizon for all parts of the
network.

The stage-wise variation in \(\bar{\rho}_{t}^{(\mathrm{stage})}\) and
\(\bar{\lambda}_{t}^{(\mathrm{stage})}\) does not imply that the mixing
coefficient \(\beta\) is adaptive. In our formulation, \(\beta\) is fixed and
globally controls the interpolation between the current clipped gradient sum and
the private-release-history memory branch. The adaptive quantities are the
spectral exponent, spectral deviation, tempering coefficient, and effective
memory depth.

\begin{figure}[t]
    \centering
    \begin{subfigure}[t]{0.485\linewidth}
        \centering
        \includegraphics[width=\linewidth]{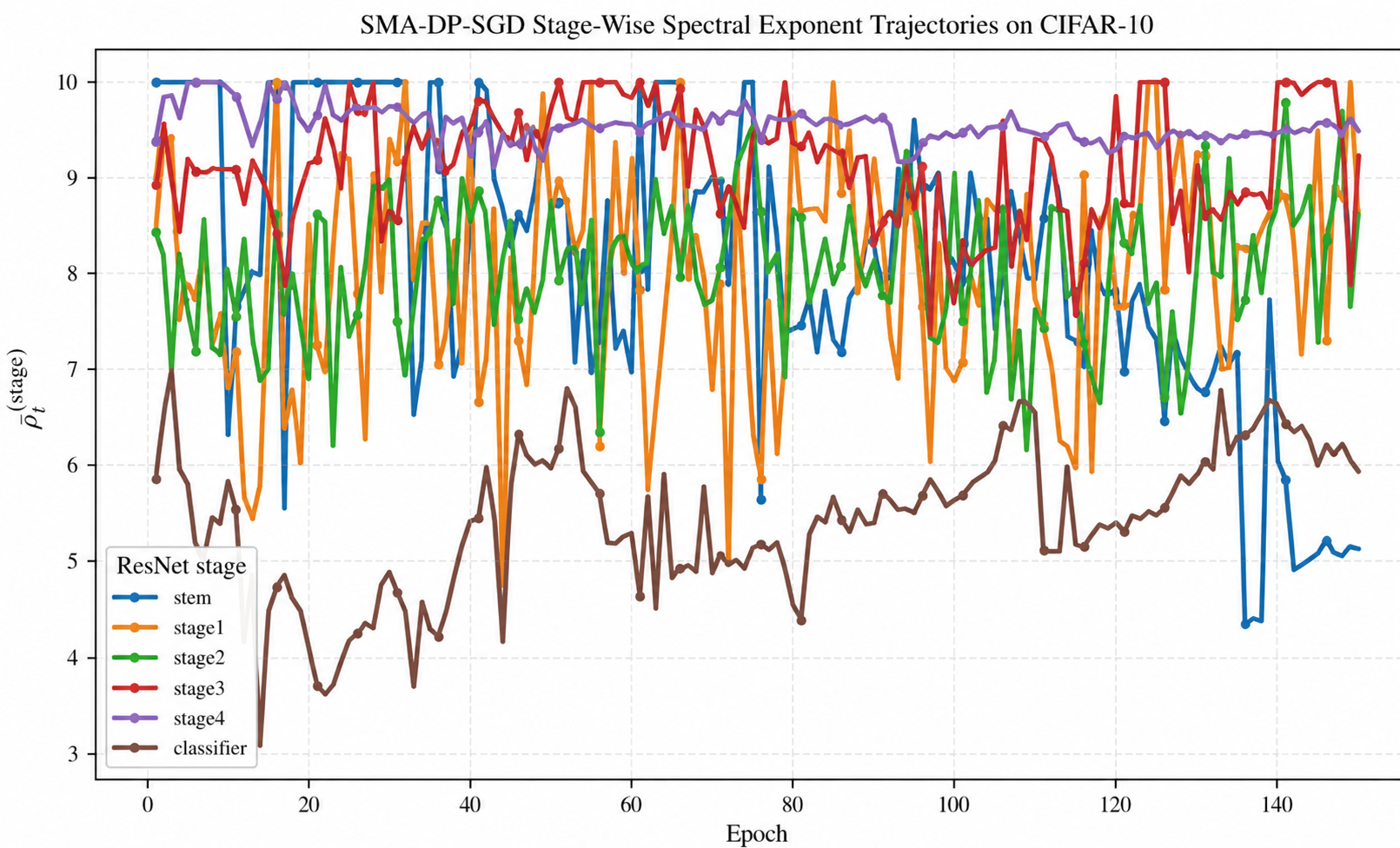}
        \caption{Stage-wise spectral exponent
        \(\bar{\rho}_{t}^{(\mathrm{stage})}\).}
        \label{fig:stage_wise_rho_cifar10}
    \end{subfigure}
    \hfill
    \begin{subfigure}[t]{0.485\linewidth}
        \centering
        \includegraphics[width=\linewidth]{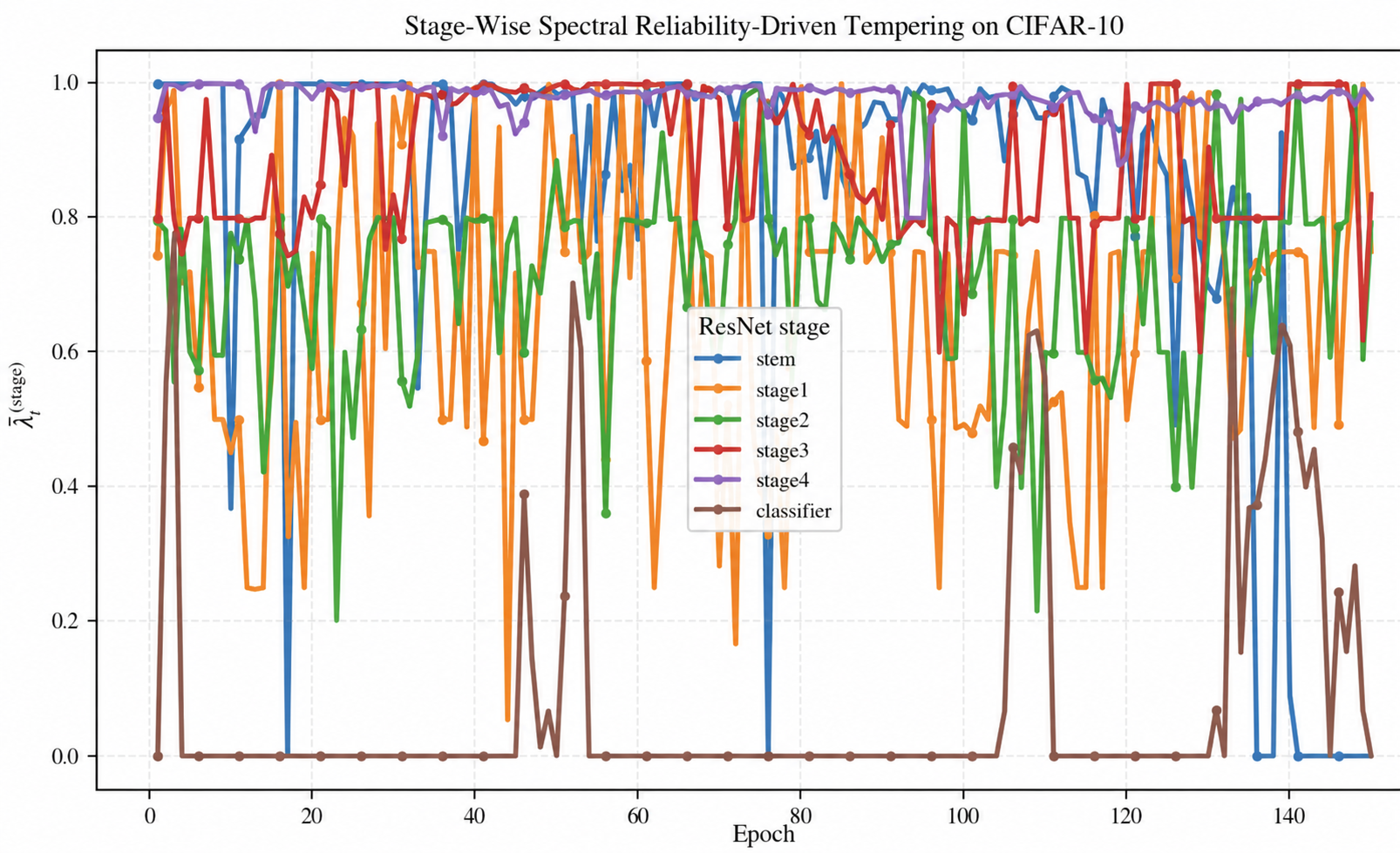}
        \caption{Stage-wise tempering coefficient
        \(\bar{\lambda}_{t}^{(\mathrm{stage})}\).}
        \label{fig:stage_wise_lambda_cifar10}
    \end{subfigure}
    \caption{
    Stage-wise spectral dynamics and reliability-driven tempering of
    SMA-DP-SGD on CIFAR-10. Panel~\subref{fig:stage_wise_rho_cifar10} reports
    the stage-wise average spectral exponent over convolutional or linear
    parameter groups within each ResNet-18 component. Panel~\subref{fig:stage_wise_lambda_cifar10}
    reports the corresponding stage-wise average tempering coefficient.
    }
    \label{fig:stage_wise_spectral_and_tempering_cifar10}
\end{figure}

\subsection{Final Accuracy Comparison on CIFAR-10}
\label{app:cifar10_final_accuracy_comparison}

To evaluate the stability and final predictive performance of SMA-DP-SGD, we
compare its final test accuracy against several differentially private
optimizer baselines on CIFAR-10. Each method is evaluated over three
independent runs, and we report the mean final test accuracy, sample standard
deviation, and a two-sided \(95\%\) confidence interval.

As shown in Table~\ref{tab:cifar10-final-accuracy-mean-std-ci-all},
SMA-DP-SGD achieves the highest final test accuracy among all compared methods.
The proposed method obtains \(0.3463 \pm 0.0003\), with a narrow \(95\%\)
confidence interval of \([0.3456,\,0.3471]\). This indicates not only a higher
mean final accuracy, but also substantially lower run-to-run variability
compared with several DP optimizer baselines.

The improvement is consistent with the intended role of the spectral
memory-aware mechanism. In this experiment, \(\beta=0.95\), \(\alpha=0.70\),
and \(K=4\), meaning that the method remains close to the current private
gradient update while allowing a controlled fractional-memory contribution.

\begin{table}[t]
\centering
\caption{
Final CIFAR-10 test accuracy over three independent runs. We report mean,
sample standard deviation, and two-sided \(95\%\) confidence intervals.
}
\label{tab:cifar10-final-accuracy-mean-std-ci-all}
\scriptsize
\setlength{\tabcolsep}{4pt}
\begin{tabular}{lcc}
\toprule
Algorithm & Mean \(\pm\) std & 95\% CI \\
\midrule
DP-SGD & \(0.3382 \pm 0.0031\) & \([0.3306,\,0.3458]\) \\
DP-Adam & \(0.3295 \pm 0.0095\) & \([0.3058,\,0.3532]\) \\
DP-AdamW & \(0.3203 \pm 0.0070\) & \([0.3029,\,0.3378]\) \\
DP-IS & \(0.3278 \pm 0.0023\) & \([0.3222,\,0.3334]\) \\
DP-SAM & \(0.3215 \pm 0.0105\) & \([0.2954,\,0.3476]\) \\
DP-SAT & \(0.3215 \pm 0.0105\) & \([0.2954,\,0.3476]\) \\
DP-Adam-AC & \(0.3073 \pm 0.0139\) & \([0.2728,\,0.3419]\) \\
SMA-DP-SGD & \(0.3463 \pm 0.0003\) & \([0.3456,\,0.3471]\) \\
\bottomrule
\end{tabular}
\end{table}

\subsection{Diagnostic Effect of the Mixing Parameter on Marginal Privacy Cost}
\label{app:cifar10_beta_ablation_privacy_budget}

We further study the effect of the fixed current-memory mixing parameter
\(\beta\) on the marginal privacy-cost trajectory of SMA-DP-SGD on CIFAR-10.
This diagnostic uses the marginal group-wise ratio \(\sigma/\beta\) to
visualize how \(\beta\) affects the current-query noise-to-sensitivity ratio
for an individual group. It should not be interpreted as the formal full-model
privacy guarantee; the formal guarantee follows the conservative joint
accountant in Appendix~\ref{app:sma_dp_sgd_theory}.

Figure~\ref{fig:cifar10_beta_ablation_epsilon_epoch} shows the resulting
marginal accumulated privacy-cost trajectory as a function of epoch for
different fixed values of \(\beta\). The results follow the expected monotonic
pattern: \(\beta=1.00\), corresponding to the DP-SGD special case with no memory
contribution, has the largest marginal privacy cost. As \(\beta\) decreases,
the marginal privacy cost becomes progressively smaller.

This reduction is not a free privacy improvement: smaller \(\beta\) also
attenuates the direct current-gradient signal and increases reliance on
historical private releases. Therefore, \(\beta\) should be selected by
considering both utility and privacy.

\begin{figure}[t]
    \centering
    \includegraphics[width=0.92\linewidth]{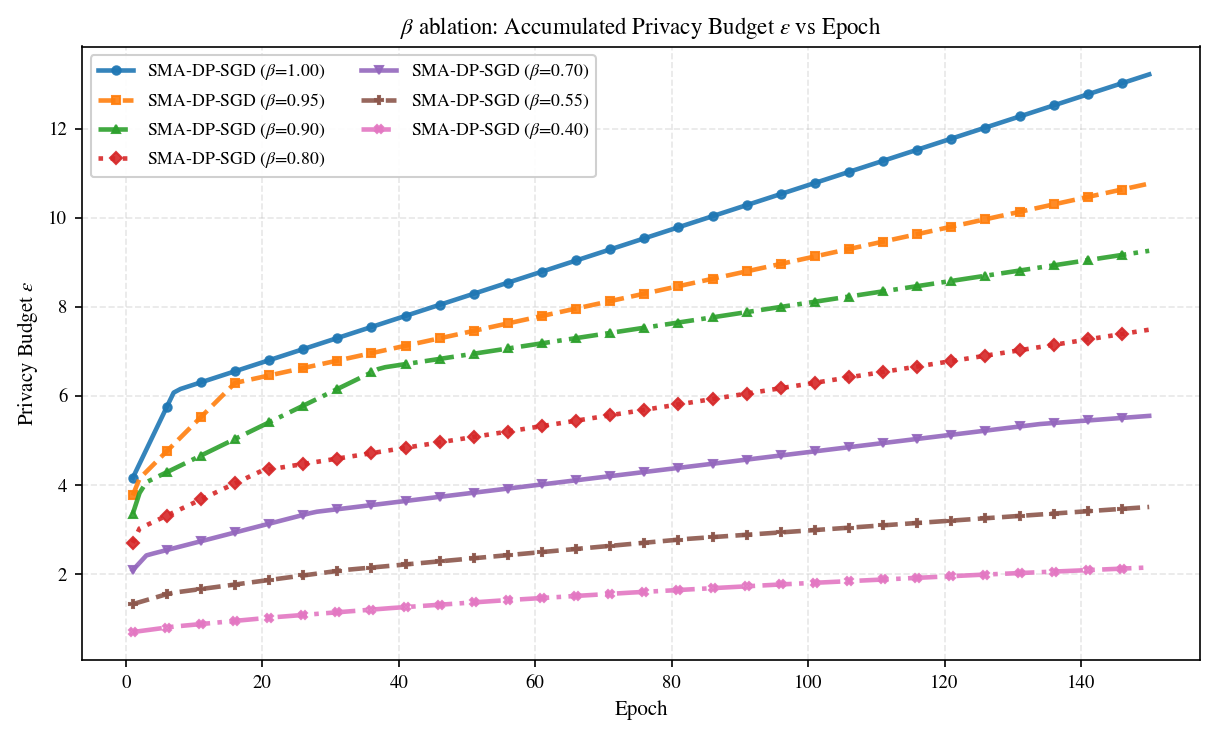}
    \caption{
    Diagnostic effect of the fixed current-memory mixing parameter \(\beta\) on
    the marginal group-wise accumulated privacy cost of SMA-DP-SGD on CIFAR-10.
    The curves use the marginal effective ratio \(\sigma/\beta\). The formal
    full-step privacy guarantee follows the conservative joint accountant in
    Appendix~\ref{app:sma_dp_sgd_theory}.
    }
    \label{fig:cifar10_beta_ablation_epsilon_epoch}
\end{figure}

\subsection{Spectral Interval Sensitivity of Effective Fractional Memory}
\label{app:cifar10_spectral_interval_memory_diagnostics}

To assess the sensitivity of the memory mechanism to the spectral reliability
interval, we evaluate SMA-DP-SGD under different choices of \(I_\rho\). The
interval \(I_\rho\) determines which spectral exponent values are treated as
reliable by the memory controller. Intervals that include the observed
layer-wise spectral exponents more often tend to produce weaker tempering and
therefore longer effective memory, whereas intervals farther from the observed
spectral regime induce stronger tempering and shorter memory depth.

\begin{table}[t]
\centering
\caption{
Sensitivity of SMA-DP-SGD to the spectral reliability interval \(I_\rho\) on
CIFAR-10. The table reports the mean effective memory depth and the mean
memory-branch ratio for different choices of \(I_\rho\).
}
\label{tab:cifar10-spectral-interval-memory-diagnostics}
\setlength{\tabcolsep}{4pt}
\begin{tabular}{lcc}
\toprule
Spectral interval \(I_\rho\) & Mean \(D_{\mathrm{eff}}\) & Mean memory ratio \\
\midrule
\([1,3]\) & \(2.7636\) & \(0.0065\) \\
\([2,4]\) & \(3.0902\) & \(0.0136\) \\
\([2,6]\) & \(3.2079\) & \(0.0133\) \\
\([3,5]\) & \(2.7487\) & \(0.0113\) \\
\([4,6]\) & \(1.9857\) & \(0.0072\) \\
\([5,7]\) & \(1.6749\) & \(0.0027\) \\
\bottomrule
\end{tabular}%
\end{table}

As reported in Table~\ref{tab:cifar10-spectral-interval-memory-diagnostics},
the effective memory depth varies noticeably across spectral intervals. The
interval \([2,6]\) yields the largest mean effective memory depth, followed
closely by \([2,4]\), indicating that these intervals allow the
private-release-history memory branch to retain a longer temporal horizon. In
contrast, intervals such as \([4,6]\) and \([5,7]\) produce shorter effective
memory depths, suggesting stronger spectral tempering and faster decay of older
private releases. The mean memory ratio remains small for all intervals,
indicating that the memory branch acts as a bounded correction rather than
dominating the private update.

\subsection{Runtime Comparison on CIFAR-10}
\label{app:cifar10_runtime_comparison}

To assess the computational overhead of SMA-DP-SGD, we compare the wall-clock
runtime of the proposed method against differentially private optimizer
baselines on CIFAR-10. Runtime is measured for the full training procedure
under the same dataset size, model configuration, privacy parameters, and
number of epochs.

As shown in Table~\ref{tab:cifar10-relative-runtime}, DP-SGD, DP-Adam, and
DP-AdamW have nearly identical runtimes. DP-IS incurs a modest overhead of
\(1.07\times\), while DP-SAM and DP-SAT require approximately \(1.86\times\)
and \(1.81\times\) the runtime of DP-SGD, respectively. SMA-DP-SGD has the
largest runtime in this comparison, requiring \(2.94\times\) the runtime of
DP-SGD. This overhead comes from the layer-wise spectral diagnostics, the
maintenance of the private release history, and the computation of the tempered
fractional memory branch for each parameter group.

\begin{table}[t]
\centering
\caption{
Relative runtime comparison on CIFAR-10. Runtime is measured in seconds and
normalized by DP-SGD.
}
\label{tab:cifar10-relative-runtime}
\setlength{\tabcolsep}{3pt}
\begin{tabular}{lccc}
\toprule
Algorithm & Mean (s) & Mean \(\pm\) std (s) & Rel. runtime \\
\midrule
DP-SGD & \( \ \ 92.66\) & \( \ \ 92.66 \pm 0.00\) & \(1.00\times\) \\
DP-Adam & \( \ \ 92.75\) & \( \ \ 92.75 \pm 0.00\) & \(1.00\times\) \\
DP-AdamW & \(\ \ 91.83\) & \( \ \ 91.83 \pm 0.00\) & \(0.99\times\) \\
DP-IS & \(\ \ 98.84\) & \( \ \ 98.84 \pm 0.00\) & \(1.07\times\) \\
DP-SAM & \(172.25\) & \(172.25 \pm 0.00\) & \(1.86\times\) \\
DP-SAT & \(167.63\) & \(167.63 \pm 0.00\) & \(1.81\times\) \\
SMA-DP-SGD (\(\beta=0.55\), \(\alpha=0.90\), \(K=8\))
& \(272.52\) & \(272.52 \pm 0.00\) & \(2.94\times\) \\
\bottomrule
\end{tabular}%
\end{table}

\section{Theoretical Analysis and Interpretation of SMA-DP-SGD}
\label{app:sma_dp_sgd_theory}

This section analyzes SMA-DP-SGD from the perspectives of recursive-query
sensitivity, conservative joint R\'enyi differential privacy (RDP) accounting,
adaptive composition, mechanism-level interpretation, memory--noise
decomposition, and limiting regimes. The analysis corresponds to the
private-release-history version of SMA-DP-SGD.

For each parameter group \(g\), the recursive query is
\begin{equation}
r_t^{(g)}(D;m_t,\mathcal H_t)
=
\beta s_t^{(g)}(D;m_t)
+
b_t^{(g)}(\mathcal H_t),
\label{eq:theory_recursive_query_global}
\end{equation}
where \(s_t^{(g)}(D;m_t)\) is the current group-wise clipped subsampled sum and
\begin{equation}
b_t^{(g)}(\mathcal H_t)
=
(1-\beta)\omega_t
\Gamma_t^{(g)}
\Psi_t^{(g)}
\nu_{t-1}^{(g)}
\label{eq:theory_memory_branch_global}
\end{equation}
is the memory branch computed from the prior private release history.

The global private release history is
\begin{equation}
\mathcal H_t
=
\left(
\mathcal H_t^{(1)},\ldots,\mathcal H_t^{(G)}
\right),
\qquad
\mathcal H_t^{(g)}
=
\left(
\tilde s_0^{(g)},\ldots,\tilde s_{t-1}^{(g)}
\right).
\label{eq:theory_global_history}
\end{equation}
This private release history is the private transcript in the adaptive
composition view. We condition on the global history rather than only the
group-wise history because the current model state \(\theta_t\), spectral
diagnostics, private EMA trends, and group-wise gradients may depend on all
previous group-wise private releases. The central private-release-history
property is that
\[
\nu_{t-1}^{(g)},\quad
\Gamma_t^{(g)},\quad
\Psi_t^{(g)},\quad
\lambda_t^{(g)}(I_\rho),\quad
\mu_{t-1}^{(g)}
\]
are deterministic functions of \(\mathcal H_t\), the model state induced by
previous private releases, and public hyperparameters. They do not depend on
the current clipped sum \(s_t^{(g)}\). Consequently, once \(\mathcal H_t\) is
fixed, the entire memory branch \(b_t^{(g)}(\mathcal H_t)\) is fixed, and the
only newly data-dependent component of \(r_t^{(g)}\) is the scaled clipped sum
\(\beta s_t^{(g)}\).

\subsection{Standing Assumptions and Notation}
\label{subsec:theory_assumptions}

\paragraph{Add/remove adjacency.}
We use add/remove adjacency. Two datasets \(D\) and \(D'\) are adjacent,
denoted \(D\sim D'\), if one can be obtained from the other by adding or
removing one example.

\paragraph{Poisson subsampling.}
At private step \(t\), a Poisson mask \(m_t=(m_{t,1},\ldots,m_{t,N})\) is drawn
with
\[
m_{t,i}\sim\mathrm{Bernoulli}(q_t)
\]
independently across examples. The sampled set is
\[
S_t=\{i:m_{t,i}=1\}.
\]
The same sampling event is used for the joint vector release across all
parameter groups at step \(t\).

\paragraph{Group-wise clipping.}
For each group \(g\), the per-example gradient is clipped as
\begin{equation}
\bar g_t^{(g)}(x_i)
=
\frac{
g_t^{(g)}(x_i)
}{
\max\left(
1,
\frac{\norm{g_t^{(g)}(x_i)}_2}{C^{(g)}}
\right)
},
\label{eq:theory_group_clipping}
\end{equation}
so that
\begin{equation}
\norm{\bar g_t^{(g)}(x_i)}_2\le C^{(g)}.
\label{eq:theory_group_clip_bound}
\end{equation}
The group-wise clipped subsampled sum is
\begin{equation}
s_t^{(g)}(D;m_t)
=
\sum_{i=1}^N m_{t,i}\bar g_t^{(g)}(x_i).
\label{eq:theory_group_clipped_sum}
\end{equation}

\paragraph{Private-release-history memory branch.}
For each group \(g\), the memory branch is
\[
b_t^{(g)}(\mathcal H_t)
=
(1-\beta)\omega_t
\Gamma_t^{(g)}
\Psi_t^{(g)}
\nu_{t-1}^{(g)}.
\]
All quantities inside this branch are deterministic functions of the global
prior private release history \(\mathcal H_t\), the current model state induced
by previous private releases, and public hyperparameters. In particular, the
private-history alignment gate \(\Gamma_t^{(g)}\) and norm scale
\(\Psi_t^{(g)}\) do not depend on the current clipped sum \(s_t^{(g)}\).

\paragraph{Shared and group-wise memory parameters.}
The fractional order \(\alpha\in(0,1]\) and memory window \(K\) are shared
across parameter groups. The number of available lags at step \(t\) is therefore
\[
M_t=\min(K-1,t).
\]
Group-wise adaptivity of the memory kernel enters through the spectral
tempering coefficient \(\lambda_t^{(g)}(I_\rho)\), which depends on the
group-wise spectral exponent \(\rho_t^{(g)}\).

\paragraph{Gaussian perturbation.}
For each group \(g\), the group-wise release is
\begin{equation}
\tilde s_t^{(g)}
=
r_t^{(g)}
+
Z_t^{(g)},
\qquad
Z_t^{(g)}
\sim
\mathcal N\left(
0,
\sigma_{t,g}^2(C^{(g)})^2 I_g
\right),
\label{eq:theory_group_release}
\end{equation}
where \(I_g\) is the identity matrix with the dimension of group \(g\). Noise
is independent across groups conditional on the current step.

\subsection{Fixed-Mask Group-Wise Clipped-Sum Sensitivity}
\label{subsec:theory_clipped_sum_sensitivity}

We first recall the fixed-mask sensitivity of the group-wise clipped sum.

\begin{lemma}[Fixed-mask sensitivity of the group-wise clipped sum]
\label{lem:theory_group_sum_sensitivity}
Fix an iteration \(t\), a group \(g\), a model state \(\theta_t\), and a
sampling mask \(m_t\). Under add/remove adjacency,
\begin{equation}
\Delta_s^{(g)}(m_t)
:=
\sup_{D\sim D'}
\norm{
s_t^{(g)}(D;m_t)-s_t^{(g)}(D';m_t)
}_2
\le
C^{(g)}.
\label{eq:theory_group_sum_sensitivity}
\end{equation}
\end{lemma}

\begin{proof}
Under add/remove adjacency, the datasets \(D\) and \(D'\) differ in at most one
example. For a fixed sampling mask \(m_t\), the two clipped sums can therefore
differ in at most one clipped per-example contribution in group \(g\). By
Eq.~\eqref{eq:theory_group_clip_bound}, this contribution has norm at most
\(C^{(g)}\). Hence
\[
\norm{
s_t^{(g)}(D;m_t)-s_t^{(g)}(D';m_t)
}_2
\le
C^{(g)}.
\]
Taking the supremum over adjacent datasets gives the result.
\end{proof}

The fixed-mask argument gives a conditional sensitivity bound. The randomness
of Poisson subsampling, and the corresponding privacy amplification, are
accounted for later through the subsampled Gaussian RDP bound.

\subsection{Private-History-Conditioned Recursive-Query Sensitivity}
\label{subsec:theory_recursive_sensitivity}

The next result formalizes why the private-release-history construction yields
a clean sensitivity bound.

\begin{proposition}[Global private-history-conditioned memory branch]
\label{prop:theory_global_history_memory_fixed}
Fix an iteration \(t\), a group \(g\), a realizable global prior private release
history \(\mathcal H_t\), and public hyperparameters. Then
\(b_t^{(g)}(\mathcal H_t)\) is fixed under this conditioning. Consequently, the
only newly data-dependent term in
\[
r_t^{(g)}(D;m_t,\mathcal H_t)
=
\beta s_t^{(g)}(D;m_t)
+
b_t^{(g)}(\mathcal H_t)
\]
is the scaled clipped sum \(\beta s_t^{(g)}(D;m_t)\).
\end{proposition}

\begin{proof}
By the private-release-history design, the fractional memory state
\(\nu_{t-1}^{(g)}\), private EMA trend \(\mu_{t-1}^{(g)}\), spectral tempering
coefficient \(\lambda_t^{(g)}(I_\rho)\), alignment gate \(\Gamma_t^{(g)}\), and
norm scale \(\Psi_t^{(g)}\) are deterministic functions of the global prior
private release history \(\mathcal H_t\), the model state induced by previous
private releases, and public hyperparameters. Therefore, once \(\mathcal H_t\)
is fixed, all these quantities are fixed. Their product in
Eq.~\eqref{eq:theory_memory_branch_global} is also fixed. Hence the only newly
data-dependent term in the recursive query is
\(\beta s_t^{(g)}(D;m_t)\).
\end{proof}

\begin{definition}[Private-history- and mask-conditioned group-wise recursive-query sensitivity]
\label{def:theory_group_recursive_sensitivity}
For fixed \(\mathcal H_t\) and \(m_t\), define
\begin{equation}
\Delta_r^{(g)}(m_t,\mathcal H_t)
=
\sup_{D\sim D'}
\norm{
r_t^{(g)}(D;m_t,\mathcal H_t)
-
r_t^{(g)}(D';m_t,\mathcal H_t)
}_2.
\label{eq:theory_group_recursive_sensitivity_def}
\end{equation}
\end{definition}

\begin{theorem}[Conditional sensitivity of SMA-DP-SGD]
\label{thm:theory_group_recursive_sensitivity}
Fix an iteration \(t\), a group \(g\), a realizable global prior private release
history \(\mathcal H_t\), and a sampling mask \(m_t\). Under add/remove
adjacency,
\begin{equation}
\Delta_r^{(g)}(m_t,\mathcal H_t)
\le
\beta C^{(g)}.
\label{eq:theory_group_recursive_sensitivity_bound}
\end{equation}
\end{theorem}

\begin{proof}
For adjacent datasets \(D\sim D'\), using
Eq.~\eqref{eq:theory_recursive_query_global},
\[
\begin{aligned}
&
r_t^{(g)}(D;m_t,\mathcal H_t)
-
r_t^{(g)}(D';m_t,\mathcal H_t)
\\
&=
\left[
\beta s_t^{(g)}(D;m_t)+b_t^{(g)}(\mathcal H_t)
\right]
-
\left[
\beta s_t^{(g)}(D';m_t)+b_t^{(g)}(\mathcal H_t)
\right]
\\
&=
\beta
\left[
s_t^{(g)}(D;m_t)-s_t^{(g)}(D';m_t)
\right]
+
\underbrace{
b_t^{(g)}(\mathcal H_t)-b_t^{(g)}(\mathcal H_t)
}_{=0}
\\
&=
\beta
\left[
s_t^{(g)}(D;m_t)-s_t^{(g)}(D';m_t)
\right].
\end{aligned}
\]
Taking norms and applying Lemma~\ref{lem:theory_group_sum_sensitivity} gives
\[
\norm{
r_t^{(g)}(D;m_t,\mathcal H_t)
-
r_t^{(g)}(D';m_t,\mathcal H_t)
}_2
\le
\beta C^{(g)}.
\]
Taking the supremum over adjacent datasets proves the claim.
\end{proof}

\begin{remark}[Why current-gradient-dependent gates are excluded]
\label{rem:theory_history_only_required}
The cancellation in Theorem~\ref{thm:theory_group_recursive_sensitivity}
depends critically on the memory branch being computed from the prior private
release history. If \(\Gamma_t^{(g)}\) or \(\Psi_t^{(g)}\) depended on the
current clipped sum \(s_t^{(g)}\), then the memory branch would generally
differ between adjacent datasets, and the bound
\(\Delta_r^{(g)}\le \beta C^{(g)}\) would not follow from the above argument.
\end{remark}

\subsection{Marginal Group-Wise Ratio versus Joint Full-Step Accountant}
\label{subsec:theory_marginal_vs_joint}

For a single group \(g\) viewed in isolation, the conditional sensitivity from
Theorem~\ref{thm:theory_group_recursive_sensitivity} is
\(\beta C^{(g)}\), while the Gaussian standard deviation is
\(\sigma_{t,g}C^{(g)}\). Thus, the marginal group-wise effective
noise-to-sensitivity ratio is
\begin{equation}
\frac{\sigma_{t,g}C^{(g)}}{\beta C^{(g)}}
=
\frac{\sigma_{t,g}}{\beta}.
\label{eq:theory_marginal_ratio}
\end{equation}
This ratio is useful for interpreting group-wise diagnostics and the role of
\(\beta\): decreasing \(\beta\) increases the marginal noise-to-sensitivity
ratio for the current group-wise query.

However, the marginal ratio in Eq.~\eqref{eq:theory_marginal_ratio} is not, by
itself, a full-model privacy guarantee when all groups are released jointly.
Because SMA-DP-SGD releases the vector of all group-wise noisy quantities under
the same Poisson subsampling event, the formal full-step privacy guarantee uses
the conservative joint accountant derived below. Therefore, empirical curves
based on the \(\sigma/\beta\) ratio should be interpreted as marginal
group-wise privacy-cost diagnostics unless \(G=1\) or an explicit composition
over groups is performed.

\subsection{Joint Per-Step RDP Accounting}
\label{subsec:theory_joint_rdp}

The group-wise sensitivity bound is useful, but the full per-step mechanism
releases all groups under the same Poisson subsampling event. Therefore, the
main privacy statement is formulated for the joint vector release.

Define the joint recursive query and joint release as
\begin{equation}
r_t(D;m_t,\mathcal H_t)
=
\left(
r_t^{(1)}(D;m_t,\mathcal H_t),\ldots,
r_t^{(G)}(D;m_t,\mathcal H_t)
\right),
\label{eq:theory_joint_query}
\end{equation}
and
\begin{equation}
\tilde s_t
=
\left(
\tilde s_t^{(1)},\ldots,\tilde s_t^{(G)}
\right)
=
r_t(D;m_t,\mathcal H_t)+Z_t.
\label{eq:theory_joint_release}
\end{equation}
The noise vector \(Z_t\) has block-diagonal covariance
\begin{equation}
\Sigma_t
=
\operatorname{diag}
\left(
\sigma_{t,1}^2(C^{(1)})^2I_1,\ldots,
\sigma_{t,G}^2(C^{(G)})^2I_G
\right).
\label{eq:theory_joint_covariance}
\end{equation}

\begin{definition}[Whitened joint sensitivity]
\label{def:theory_whitened_sensitivity}
For fixed \(\mathcal H_t\) and \(m_t\), define the whitened joint sensitivity
\begin{equation}
\Delta_{\mathrm{white},t}
=
\sup_{D\sim D'}
\norm{
\Sigma_t^{-1/2}
\left[
r_t(D;m_t,\mathcal H_t)
-
r_t(D';m_t,\mathcal H_t)
\right]
}_2.
\label{eq:theory_whitened_sensitivity_def}
\end{equation}
\end{definition}

\begin{lemma}[Bound on the whitened joint sensitivity]
\label{lem:theory_whitened_sensitivity_bound}
For fixed \(\mathcal H_t\) and \(m_t\),
\begin{equation}
\Delta_{\mathrm{white},t}
\le
\beta
\left(
\sum_{g=1}^G
\sigma_{t,g}^{-2}
\right)^{1/2}.
\label{eq:theory_whitened_sensitivity_bound}
\end{equation}
\end{lemma}

\begin{proof}
For adjacent datasets \(D\sim D'\), let
\[
\delta_t^{(g)}
=
r_t^{(g)}(D;m_t,\mathcal H_t)
-
r_t^{(g)}(D';m_t,\mathcal H_t).
\]
By Theorem~\ref{thm:theory_group_recursive_sensitivity},
\[
\norm{\delta_t^{(g)}}_2\le \beta C^{(g)}.
\]
The squared whitened norm of the joint difference is
\[
\begin{aligned}
\norm{
\Sigma_t^{-1/2}
\left[
r_t(D;m_t,\mathcal H_t)
-
r_t(D';m_t,\mathcal H_t)
\right]
}_2^2
&=
\sum_{g=1}^G
\frac{
\norm{\delta_t^{(g)}}_2^2
}{
\sigma_{t,g}^2(C^{(g)})^2
}
\le
\sum_{g=1}^G
\frac{
\beta^2(C^{(g)})^2
}{
\sigma_{t,g}^2(C^{(g)})^2
} 
=
\beta^2
\sum_{g=1}^G
\sigma_{t,g}^{-2}.
\end{aligned}
\]
Taking square roots and then the supremum over adjacent datasets gives the
claim.
\end{proof}

Define the conservative effective joint noise-to-sensitivity ratio
\begin{equation}
\sigma_{\mathrm{eff},t}
=
\frac{
1
}{
\beta
\left(
\sum_{g=1}^G
\sigma_{t,g}^{-2}
\right)^{1/2}
}.
\label{eq:theory_sigma_eff}
\end{equation}
If all groups use the same noise multiplier \(\sigma_{t,g}=\sigma_t\), then
\begin{equation}
\sigma_{\mathrm{eff},t}
=
\frac{\sigma_t}{\beta\sqrt{G}}.
\label{eq:theory_sigma_eff_equal_noise}
\end{equation}

\begin{remark}[Conservativeness of the joint bound]
\label{rem:theory_joint_bound_conservative}
The bound in Lemma~\ref{lem:theory_whitened_sensitivity_bound} is a valid upper
bound obtained from the individual group-wise sensitivity bounds. It may be
conservative if the true concatenated contribution of a single example across
groups is smaller than the worst-case combination of all group-wise clipping
bounds. We do not claim that this joint accounting is tight. In particular, the
joint effective ratio in Eq.~\eqref{eq:theory_sigma_eff} can be substantially
smaller than the marginal group-wise ratio \(\sigma_{t,g}/\beta\).
\end{remark}

Let
\[
\varepsilon_{\mathrm{SGM}}(\lambda_{\mathrm R};q,\sigma)
\]
denote any valid RDP upper bound at order \(\lambda_{\mathrm R}>1\) for the
Poisson-subsampled Gaussian mechanism with subsampling probability \(q\) and
noise-to-sensitivity ratio \(\sigma\).

\begin{theorem}[Joint per-step RDP guarantee]
\label{thm:theory_joint_per_step_rdp}
Fix an iteration \(t\) and a realizable global prior private release history
\(\mathcal H_t\). The joint release
\[
\tilde s_t
=
r_t(D;m_t,\mathcal H_t)+Z_t,
\qquad
Z_t\sim\mathcal N(0,\Sigma_t),
\]
with Poisson subsampling probability \(q_t\), satisfies
\begin{equation}
\left(
\lambda_{\mathrm R},
\varepsilon_{\mathrm{SGM}}
\left(
\lambda_{\mathrm R};
q_t,
\sigma_{\mathrm{eff},t}
\right)
\right)\text{-RDP}
\label{eq:theory_joint_per_step_rdp}
\end{equation}
for every \(\lambda_{\mathrm R}>1\).
\end{theorem}

\begin{proof}
Conditioned on the global prior private release history \(\mathcal H_t\), all
memory branches are fixed. The current data access at step \(t\) occurs through
the same Poisson subsampling event \(m_t\) for the joint vector release.

By Lemma~\ref{lem:theory_whitened_sensitivity_bound}, the joint recursive query
has whitened sensitivity at most
\[
\Delta_{\mathrm{white},t}
\le
\beta
\left(
\sum_{g=1}^G
\sigma_{t,g}^{-2}
\right)^{1/2}.
\]
Whitening by \(\Sigma_t^{-1/2}\) transforms the anisotropic Gaussian release
into an isotropic Gaussian mechanism with identity covariance and normalized
sensitivity bounded by \(\Delta_{\mathrm{white},t}\). Hence the effective
noise-to-sensitivity ratio is at least
\[
\frac{1}{\Delta_{\mathrm{white},t}}
\ge
\frac{
1
}{
\beta
\left(
\sum_{g=1}^G
\sigma_{t,g}^{-2}
\right)^{1/2}
}
=
\sigma_{\mathrm{eff},t}.
\]
Therefore, any valid RDP upper bound for the Poisson-subsampled Gaussian
mechanism applies with effective ratio \(\sigma_{\mathrm{eff},t}\), yielding
Eq.~\eqref{eq:theory_joint_per_step_rdp}.
\end{proof}

\subsection{Adaptive Composition over Training}
\label{subsec:theory_adaptive_composition}

SMA-DP-SGD is adaptive because later model states and memory branches depend on
previous private releases. RDP adaptive composition applies because each step
conditions on the prior private release history and then releases a valid
private joint mechanism.

\begin{theorem}[Adaptive composition for SMA-DP-SGD]
\label{thm:theory_adaptive_composition}
Suppose SMA-DP-SGD is run for \(T\) private steps. For every
\(\lambda_{\mathrm R}>1\), the full release history satisfies
\begin{equation}
\left(
\lambda_{\mathrm R},
\varepsilon_{\mathrm{tot}}(\lambda_{\mathrm R})
\right)\text{-RDP},
\label{eq:theory_total_rdp_statement}
\end{equation}
where
\begin{equation}
\varepsilon_{\mathrm{tot}}(\lambda_{\mathrm R})
=
\sum_{t=0}^{T-1}
\varepsilon_{\mathrm{SGM}}
\left(
\lambda_{\mathrm R};
q_t,
\sigma_{\mathrm{eff},t}
\right).
\label{eq:theory_total_rdp}
\end{equation}
Consequently, for any \(\delta\in(0,1)\), SMA-DP-SGD satisfies
\((\varepsilon_\delta,\delta)\)-DP with
\begin{equation}
\varepsilon_\delta
=
\inf_{\lambda_{\mathrm R}>1}
\left\{
\varepsilon_{\mathrm{tot}}(\lambda_{\mathrm R})
+
\frac{\log(1/\delta)}{\lambda_{\mathrm R}-1}
\right\}.
\label{eq:theory_rdp_to_dp}
\end{equation}
\end{theorem}

\begin{proof}
By Theorem~\ref{thm:theory_joint_per_step_rdp}, the joint release at step \(t\)
satisfies
\[
\left(
\lambda_{\mathrm R},
\varepsilon_{\mathrm{SGM}}
\left(
\lambda_{\mathrm R};
q_t,
\sigma_{\mathrm{eff},t}
\right)
\right)\text{-RDP}.
\]
Although the algorithm is adaptive, this adaptivity is through prior private
releases. Adaptive composition of RDP permits summing the per-step RDP costs,
which gives Eq.~\eqref{eq:theory_total_rdp}. The standard conversion from RDP
to \((\varepsilon,\delta)\)-DP yields Eq.~\eqref{eq:theory_rdp_to_dp}.
\end{proof}

\subsection{Mechanism-Level Interpretation}
\label{subsec:theory_mechanism_level}

SMA-DP-SGD uses prior private releases to construct its memory branch. This
memory construction is post-processing of the prior private release history.
However, the current step is not merely post-processing of a standard DP-SGD
release. Instead, SMA-DP-SGD forms a new history-dependent recursive query and
then privatizes that query with fresh Gaussian noise.

\begin{definition}[Standard and SMA-DP-SGD group-wise queries]
\label{def:theory_standard_sma_queries}
At step \(t\), define the standard group-wise DP-SGD clipped-sum query as
\begin{equation}
q_t^{\mathrm{DP\text{-}SGD},(g)}(D;m_t)
=
s_t^{(g)}(D;m_t),
\label{eq:theory_standard_dpsgd_query}
\end{equation}
and the SMA-DP-SGD recursive query as
\begin{equation}
q_t^{\mathrm{SMA},(g)}(D;m_t,\mathcal H_t)
=
\beta s_t^{(g)}(D;m_t)
+
b_t^{(g)}(\mathcal H_t).
\label{eq:theory_sma_query}
\end{equation}
\end{definition}

\begin{proposition}[Mechanism-level modification]
\label{prop:theory_mechanism_level}
Suppose \(0<\beta<1\) and \(b_t^{(g)}(\mathcal H_t)\neq 0\). Then the
SMA-DP-SGD query differs structurally from the standard group-wise DP-SGD
clipped-sum query. In particular, SMA-DP-SGD forms a history-dependent
recursive query before adding fresh Gaussian noise and is therefore not merely
post-processing of a standard DP-SGD private release.
\end{proposition}

\begin{proof}
The standard group-wise DP-SGD query is
\[
q_t^{\mathrm{DP\text{-}SGD},(g)}(D;m_t)
=
s_t^{(g)}(D;m_t).
\]
The SMA-DP-SGD query is
\[
q_t^{\mathrm{SMA},(g)}(D;m_t,\mathcal H_t)
=
\beta s_t^{(g)}(D;m_t)
+
b_t^{(g)}(\mathcal H_t).
\]
When \(0<\beta<1\) and \(b_t^{(g)}(\mathcal H_t)\neq0\), this query contains a
nonzero history-dependent memory branch and a scaled current clipped sum.
Therefore, the query object is structurally different from the standard
DP-SGD clipped-sum query.

Moreover, SMA-DP-SGD releases
\[
\tilde s_t^{(g)}
=
q_t^{\mathrm{SMA},(g)}(D;m_t,\mathcal H_t)
+
Z_t^{(g)},
\]
so fresh Gaussian noise is applied after the recursive query is formed. Thus,
privacy follows from fresh Gaussian perturbation of the history-dependent query
together with the conditional sensitivity bound, not merely from post-processing
a standard DP-SGD release.
\end{proof}

\subsection{Signal--Memory--Noise Decomposition}
\label{subsec:theory_signal_memory_noise}

The private-release-history memory branch reuses previous noisy private
releases. This can carry useful optimization signal, but it also carries
inherited Gaussian noise.

In this analysis, the fractional order \(\alpha\) and memory window \(K\) are
shared across parameter groups. Group-wise adaptivity of the memory kernel
enters through \(\lambda_t^{(g)}(I_\rho)\), which depends on the group-wise
spectral exponent \(\rho_t^{(g)}\). Let
\[
M_t=\min(K-1,t).
\]
For \(M_t\ge1\), the fractional memory state is
\begin{equation}
\nu_{t-1}^{(g)}
=
\sum_{j=1}^{M_t}
\hat a_{t,j}^{(g)}
\tilde s_{t-j}^{(g)}.
\label{eq:theory_fractional_memory_recall}
\end{equation}
Each previous private release satisfies
\[
\tilde s_{t-j}^{(g)}
=
r_{t-j}^{(g)}
+
Z_{t-j}^{(g)}.
\]

\begin{lemma}[Algebraic signal--memory--noise decomposition]
\label{lem:theory_signal_memory_noise}
For \(M_t\ge1\), the memory state admits the algebraic decomposition
\begin{equation}
\nu_{t-1}^{(g)}
=
\nu_{t-1}^{\mathrm{rec},(g)}
+
\nu_{t-1}^{\mathrm{noise},(g)},
\label{eq:theory_memory_decomposition}
\end{equation}
where
\begin{equation}
\nu_{t-1}^{\mathrm{rec},(g)}
=
\sum_{j=1}^{M_t}
\hat a_{t,j}^{(g)}
r_{t-j}^{(g)}
\label{eq:theory_memory_rec_component}
\end{equation}
and
\begin{equation}
\nu_{t-1}^{\mathrm{noise},(g)}
=
\sum_{j=1}^{M_t}
\hat a_{t,j}^{(g)}
Z_{t-j}^{(g)}.
\label{eq:theory_memory_noise_component}
\end{equation}
Consequently,
\begin{equation}
\begin{aligned}
r_t^{(g)}
&=
\beta s_t^{(g)}
+
(1-\beta)\omega_t\Gamma_t^{(g)}\Psi_t^{(g)}
\nu_{t-1}^{\mathrm{rec},(g)}
\\
&\quad+
(1-\beta)\omega_t\Gamma_t^{(g)}\Psi_t^{(g)}
\nu_{t-1}^{\mathrm{noise},(g)} .
\end{aligned}
\label{eq:theory_recursive_query_decomposition}
\end{equation}
\end{lemma}

\begin{proof}
Substituting
\[
\tilde s_{t-j}^{(g)}
=
r_{t-j}^{(g)}
+
Z_{t-j}^{(g)}
\]
into Eq.~\eqref{eq:theory_fractional_memory_recall} gives
\[
\begin{aligned}
\nu_{t-1}^{(g)}
&=
\sum_{j=1}^{M_t}
\hat a_{t,j}^{(g)}
\left(
r_{t-j}^{(g)}+Z_{t-j}^{(g)}
\right)
\\
&=
\sum_{j=1}^{M_t}
\hat a_{t,j}^{(g)}r_{t-j}^{(g)}
+
\sum_{j=1}^{M_t}
\hat a_{t,j}^{(g)}Z_{t-j}^{(g)}.
\end{aligned}
\]
This proves Eq.~\eqref{eq:theory_memory_decomposition}. Substituting this
decomposition into the recursive query gives
Eq.~\eqref{eq:theory_recursive_query_decomposition}.
\end{proof}

\begin{remark}[Algebraic, not independence-based]
\label{rem:theory_algebraic_decomposition}
The decomposition in Lemma~\ref{lem:theory_signal_memory_noise} is algebraic.
The weights \(\hat a_{t,j}^{(g)}\) are functions of the prior private release
history and may therefore be statistically dependent on previous Gaussian noise
terms. Thus, the decomposition should not be interpreted as an independence,
orthogonality, or denoising decomposition.
\end{remark}

\begin{remark}[Inherited noise]
\label{rem:theory_inherited_noise}
SMA-DP-SGD does not explicitly denoise previous private releases. The memory
branch is constructed from already-private noisy quantities, so it may contain
both inherited recursive signal and inherited Gaussian noise.
\end{remark}

\subsection{Interpretation of Spectral Tempering and Private-History Gating}
\label{subsec:theory_interpretation}

This subsection records structural properties of the spectral and memory
controls.

\begin{proposition}[Spectral deviation increases tempering]
\label{prop:theory_spectral_deviation}
For
\[
\lambda_t^{(g)}(I_\rho)
=
1-\exp\left(-c_\lambda d_t^{(g)}(I_\rho)\right),
\qquad c_\lambda>0,
\]
the tempering coefficient \(\lambda_t^{(g)}(I_\rho)\) is nondecreasing in the
spectral deviation \(d_t^{(g)}(I_\rho)\).
\end{proposition}

\begin{proof}
Differentiating gives
\[
\frac{\partial \lambda_t^{(g)}(I_\rho)}
{\partial d_t^{(g)}(I_\rho)}
=
c_\lambda\exp\left(-c_\lambda d_t^{(g)}(I_\rho)\right)\ge0.
\]
Hence larger spectral deviation yields larger or equal tempering.
\end{proof}

\begin{proposition}[Larger tempering suppresses older lags]
\label{prop:theory_tempering_suppresses_lags}
For fixed shared fractional order \(\alpha\in(0,1]\) and lag \(j\ge1\), the raw
memory coefficient
\[
a_{t,j}^{(g)}
=
(j+1)^{\alpha-1}
\exp\left(-\lambda_t^{(g)}(I_\rho)j\right)
\]
is nonincreasing in \(\lambda_t^{(g)}(I_\rho)\). Moreover, for \(j>k\), the
ratio \(a_{t,j}^{(g)}/a_{t,k}^{(g)}\) decreases as
\(\lambda_t^{(g)}(I_\rho)\) increases.
\end{proposition}

\begin{proof}
For fixed \(j\),
\[
\frac{\partial a_{t,j}^{(g)}}{\partial \lambda_t^{(g)}(I_\rho)}
=
-j a_{t,j}^{(g)}
\le0.
\]
For \(j>k\),
\[
\frac{a_{t,j}^{(g)}}{a_{t,k}^{(g)}}
=
\left(
\frac{j+1}{k+1}
\right)^{\alpha-1}
\exp\left(-\lambda_t^{(g)}(I_\rho)(j-k)\right),
\]
which decreases in \(\lambda_t^{(g)}(I_\rho)\) because \(j-k>0\).
\end{proof}

\begin{proposition}[Role of the shared fractional order]
\label{prop:theory_alpha_role}
The shared fractional order \(\alpha\in(0,1]\) controls the power-law factor
\[
(j+1)^{\alpha-1}.
\]
If \(0<\alpha<1\), this factor decreases with \(j\). If \(\alpha=1\), the
power-law factor is constant and lag dependence is governed only by exponential
tempering.
\end{proposition}

\begin{proof}
When \(0<\alpha<1\), the exponent \(\alpha-1\) is negative, so
\((j+1)^{\alpha-1}\) decreases as \(j\) grows. When \(\alpha=1\),
\((j+1)^{\alpha-1}=1\).
\end{proof}

\begin{proposition}[Bounded private-history alignment gate]
\label{prop:theory_gamma_bound}
The private-history alignment gate satisfies
\begin{equation}
0\le \Gamma_t^{(g)}\le1.
\label{eq:theory_gamma_bound}
\end{equation}
\end{proposition}

\begin{proof}
By definition,
\[
\Gamma_t^{(g)}
=
\max\left(
0,
\frac{
\inner{\mu_{t-1}^{(g)}}{\nu_{t-1}^{(g)}}
}{
\norm{\mu_{t-1}^{(g)}}_2\norm{\nu_{t-1}^{(g)}}_2+\epsilon
}
\right).
\]
The maximum with zero gives nonnegativity. By Cauchy--Schwarz,
\[
\inner{\mu_{t-1}^{(g)}}{\nu_{t-1}^{(g)}}
\le
\norm{\mu_{t-1}^{(g)}}_2\norm{\nu_{t-1}^{(g)}}_2.
\]
Since \(\epsilon>0\), the denominator is at least
\(\norm{\mu_{t-1}^{(g)}}_2\norm{\nu_{t-1}^{(g)}}_2\). Hence the cosine-like
ratio is at most one, and therefore \(\Gamma_t^{(g)}\le1\).
\end{proof}

\begin{proposition}[Bounded norm scale]
\label{prop:theory_psi_bound}
The norm scale satisfies
\begin{equation}
0\le \Psi_t^{(g)}\le\xi_{\max}.
\label{eq:theory_psi_bound}
\end{equation}
\end{proposition}

\begin{proof}
By definition,
\[
\Psi_t^{(g)}
=
\min\left(
\xi_{\max},
\frac{
\norm{\mu_{t-1}^{(g)}}_2
}{
\norm{\nu_{t-1}^{(g)}}_2+\epsilon
}
\right).
\]
Both arguments of the minimum are nonnegative, and the first argument is
\(\xi_{\max}\). Hence \(0\le\Psi_t^{(g)}\le\xi_{\max}\).
\end{proof}

\begin{proposition}[Effect of \(\beta\)]
\label{prop:theory_beta_tradeoff}
The fixed mixing coefficient \(\beta\in(0,1]\) creates a direct tradeoff:
smaller \(\beta\) reduces the current-step conditional sensitivity bound
linearly, but also attenuates the direct current clipped-sum signal by the same
factor and increases reliance on the memory branch.
\end{proposition}

\begin{proof}
By Theorem~\ref{thm:theory_group_recursive_sensitivity},
\[
\Delta_r^{(g)}\le\beta C^{(g)}.
\]
Thus decreasing \(\beta\) reduces the group-wise conditional sensitivity bound
linearly. On the other hand,
\[
r_t^{(g)}
=
\beta s_t^{(g)}
+
b_t^{(g)}(\mathcal H_t),
\]
so the direct coefficient of the current clipped sum is exactly \(\beta\).
Therefore, decreasing \(\beta\) also weakens the direct current-gradient signal
and increases the relative role of the memory branch.
\end{proof}

\begin{remark}[No free privacy improvement]
\label{rem:theory_no_free_privacy}
A smaller \(\beta\) changes the mechanism: it reduces current-query
sensitivity, but it also attenuates the current clipped-sum signal and increases
dependence on previous private releases. Thus, the privacy effect of \(\beta\)
should be interpreted jointly with its optimization effect.
\end{remark}

\subsection{Special Cases and Limiting Regimes}
\label{subsec:theory_special_cases}

\begin{proposition}[Exact reduction to group-wise DP-SGD]
\label{prop:theory_beta_one}
If \(\beta=1\), then the memory branch vanishes and SMA-DP-SGD reduces exactly
to group-wise DP-SGD:
\[
r_t^{(g)}=s_t^{(g)},
\qquad
\tilde s_t^{(g)}
=
s_t^{(g)}
+
\mathcal N\left(
0,
\sigma_{t,g}^2(C^{(g)})^2I_g
\right).
\]
\end{proposition}

\begin{proof}
If \(\beta=1\), then \(1-\beta=0\). Hence
\[
b_t^{(g)}(\mathcal H_t)
=
(1-\beta)\omega_t\Gamma_t^{(g)}\Psi_t^{(g)}\nu_{t-1}^{(g)}
=
0.
\]
Therefore \(r_t^{(g)}=s_t^{(g)}\), and substituting this into the release rule
gives the group-wise DP-SGD Gaussian release.
\end{proof}

\begin{proposition}[No-memory window]
\label{prop:theory_K_one}
If \(K=1\), then \(M_t=0\) for all \(t\), so
\[
\nu_{t-1}^{(g)}=0.
\]
The memory branch for group \(g\) therefore vanishes.
\end{proposition}

\begin{proof}
Since \(M_t=\min(K-1,t)\), setting \(K=1\) gives \(M_t=0\) for all \(t\). By
definition, the memory state is zero when no lagged terms are available.
\end{proof}

\begin{proposition}[Warm-up or alignment removal of memory]
\label{prop:theory_warmup_alignment_zero}
If \(\omega_t=0\) or \(\Gamma_t^{(g)}=0\), then the memory branch vanishes for
group \(g\) at step \(t\).
\end{proposition}

\begin{proof}
The memory branch is
\[
b_t^{(g)}(\mathcal H_t)
=
(1-\beta)\omega_t\Gamma_t^{(g)}\Psi_t^{(g)}\nu_{t-1}^{(g)}.
\]
If either \(\omega_t=0\) or \(\Gamma_t^{(g)}=0\), the product is zero.
\end{proof}

\begin{proposition}[Raw fractional memory]
\label{prop:theory_lambda_zero}
If \(\lambda_t^{(g)}(I_\rho)=0\), then
\[
a_{t,j}^{(g)}
=
(j+1)^{\alpha-1}.
\]
Thus the memory kernel reduces to a raw fractional power-law kernel before
normalization.
\end{proposition}

\begin{proof}
Substitute \(\lambda_t^{(g)}(I_\rho)=0\) into
\[
a_{t,j}^{(g)}
=
(j+1)^{\alpha-1}
\exp(-\lambda_t^{(g)}(I_\rho)j).
\]
The exponential factor becomes one.
\end{proof}

\begin{proposition}[No fractional power-law factor]
\label{prop:theory_alpha_one}
If \(\alpha=1\), then
\[
a_{t,j}^{(g)}
=
\exp(-\lambda_t^{(g)}(I_\rho)j).
\]
Thus the fractional power-law factor disappears.
\end{proposition}

\begin{proof}
If \(\alpha=1\), then \((j+1)^{\alpha-1}=(j+1)^0=1\).
\end{proof}

\begin{proposition}[Uniform memory weights]
\label{prop:theory_uniform_weights}
If \(\lambda_t^{(g)}(I_\rho)=0\) and \(\alpha=1\), then
\[
a_{t,j}^{(g)}=1
\]
for every active lag \(j\). After normalization,
\[
\hat a_{t,j}^{(g)}
=
\frac{1}{M_t},
\qquad
j=1,\ldots,M_t.
\]
\end{proposition}

\begin{proof}
Under \(\lambda_t^{(g)}(I_\rho)=0\) and \(\alpha=1\),
\[
a_{t,j}^{(g)}
=
(j+1)^0\exp(0)=1.
\]
The sum of the \(M_t\) active raw weights is \(M_t\), so normalization gives
\(\hat a_{t,j}^{(g)}=1/M_t\).
\end{proof}




\newpage
\section*{NeurIPS Paper Checklist}

\begin{enumerate}

\item {\bf Claims}
    \item[] Question: Do the main claims made in the abstract and introduction accurately reflect the paper's contributions and scope?
    \item[] Answer: \answerYes{}.
    \item[] Justification: The claims in the abstract and introduction are restricted to the proposed SMA-DP-SGD mechanism, its private-release-history fractional memory design, its WeightWatcher-inspired spectral tempering mechanism, its conditional sensitivity structure, and the empirical privacy--utility behavior observed in the reported experiments. These claims are supported by the methodology in Section~\ref{sec:methodology}, the experiments in Section~\ref{sec:experiment}, the additional results in Appendix~\ref{app:additional_experiments}, and the theoretical analysis in Appendix~\ref{app:sma_dp_sgd_theory}.

\item {\bf Limitations}
    \item[] Question: Does the paper discuss the limitations of the work performed by the authors?
    \item[] Answer: \answerYes{}.
    \item[] Justification: Section~\ref{sec:conclusion} discusses limitations including inherited noise from previous private releases, the heuristic nature of the spectral reliability interval, the additional computational overhead introduced by spectral diagnostics and memory computations, and the distinction between marginal privacy-cost diagnostics and formal full-step privacy accounting.

\item {\bf Theory assumptions and proofs}
    \item[] Question: For each theoretical result, does the paper provide the full set of assumptions and a complete and correct proof?
    \item[] Answer: \answerYes{}.
    \item[] Justification: Appendix~\ref{app:sma_dp_sgd_theory} states the theoretical assumptions and provides the formal analysis of SMA-DP-SGD. It includes add/remove adjacency, Poisson subsampling, group-wise clipping, private-release-history conditioning, recursive-query sensitivity, conservative joint RDP accounting, adaptive composition, signal--memory--noise decomposition, parameter interpretations, and limiting regimes.

\item {\bf Experimental result reproducibility}
    \item[] Question: Does the paper fully disclose all the information needed to reproduce the main experimental results of the paper to the extent that it affects the main claims and/or conclusions of the paper, regardless of whether the code and data are provided or not?
    \item[] Answer: \answerYes{}.
    \item[] Justification: Appendix~\ref{app:experimental_setup_reproducibility} provides the experimental setup and reproducibility details, including datasets, baselines, group-wise/layer-wise parameter grouping, SMA-DP-SGD configuration, repeated-run reporting, and hyperparameters specific to the proposed method. The main experimental results are reported in Section~\ref{sec:experiment}, with additional statistics and diagnostics in Appendix~\ref{app:additional_experiments}.

\item {\bf Open access to data and code}
    \item[] Question: Does the paper provide open access to the data and code, with sufficient instructions to faithfully reproduce the main experimental results, as described in supplemental material?
    \item[] Answer: \answerNo{}.
    \item[] Justification: The paper uses standard publicly available benchmark datasets, including CIFAR-100, CIFAR-10, and MNIST, and provides methodological and experimental details in Section~\ref{sec:methodology}, Section~\ref{sec:experiment}, and Appendix~\ref{app:experimental_setup_reproducibility}. The source code is not released at submission time. If the paper is published, we plan to provide the implementation and scripts upon reasonable request to support reproduction of the reported results.

\item {\bf Experimental setting/details}
    \item[] Question: Does the paper specify all the training and test details, e.g., data splits, hyperparameters, how they were chosen, type of optimizer, necessary to understand the results?
    \item[] Answer: \answerYes{}.
    \item[] Justification: The experimental setting is summarized in Section~\ref{sec:experiment} and described in detail in Appendix~\ref{app:experimental_setup_reproducibility}. The paper specifies the datasets, baselines, layer-wise instantiation of parameter groups, shared fractional order, shared memory window, fixed mixing coefficient, spectral reliability interval, tempering constant, repeated-run protocol, and statistical reporting. Additional ablation and diagnostic results are provided in Appendix~\ref{app:additional_experiments}.

\item {\bf Experiment statistical significance}
    \item[] Question: Does the paper report error bars suitably and correctly defined or other appropriate information about the statistical significance of the experiments?
    \item[] Answer: \answerYes{}.
    \item[] Justification: Appendix~\ref{app:cifar10_final_accuracy_comparison} reports final CIFAR-10 accuracy over three independent runs using mean, sample standard deviation, and two-sided 95\% confidence intervals computed using the Student-\(t\) interval. This provides statistical information about run-to-run variability for the final accuracy comparison.

\item {\bf Experiments compute resources}
    \item[] Question: For each experiment, does the paper provide sufficient information on the computer resources, type of compute workers, memory, time of execution, needed to reproduce the experiments?
    \item[] Answer: \answerYes{}.
    \item[] Justification: Runtime and relative computational overhead are reported in Appendix~\ref{app:cifar10_runtime_comparison}. The reproducibility description in Appendix~\ref{app:experimental_setup_reproducibility} states that runtime is measured under the same hardware and software environment for all optimizers in the runtime comparison. The paper reports the relative overhead of SMA-DP-SGD compared with DP-SGD and other DP baselines.

\item {\bf Code of ethics}
    \item[] Question: Does the research conducted in the paper conform, in every respect, with the NeurIPS Code of Ethics \url{https://neurips.cc/public/EthicsGuidelines}?
    \item[] Answer: \answerYes{}.
    \item[] Justification: The work is methodological and experimental. It uses public benchmark datasets and does not involve human subjects, private user data, deceptive systems, high-risk deployment, or released high-risk models. Privacy-related assumptions, limitations, and responsible-use considerations are discussed in Section~\ref{sec:conclusion} and Section~\ref{sec:broader_impact}.

\item {\bf Broader impacts}
    \item[] Question: Does the paper discuss both potential positive societal impacts and negative societal impacts of the work performed?
    \item[] Answer: \answerYes{}.
    \item[] Justification: Section~\ref{sec:broader_impact} discusses the potential positive impact of improving utility in differentially private learning for sensitive domains, as well as risks associated with incorrect privacy accounting, inappropriate privacy-budget interpretation, computational overhead, and deployment in high-stakes applications.

\item {\bf Safeguards}
    \item[] Question: Does the paper describe safeguards that have been put in place for responsible release of data or models that have a high risk for misuse, e.g., pre-trained language models, image generators, or scraped datasets?
    \item[] Answer: \answerNA{}.
    \item[] Justification: The paper does not release high-risk pretrained models, language models, image generators, scraped datasets, or dual-use datasets. The experiments use standard public benchmark datasets and focus on a differentially private optimization method.

\item {\bf Licenses for existing assets}
    \item[] Question: Are the creators or original owners of assets, e.g., code, data, models, used in the paper, properly credited and are the license and terms of use explicitly mentioned and properly respected?
    \item[] Answer: \answerYes{}.
    \item[] Justification: The paper uses standard public benchmark datasets and standard software tools for machine learning experiments. The relevant datasets and prior methods are cited in the paper. No proprietary, restricted, or newly scraped datasets are introduced.

\item {\bf New assets}
    \item[] Question: Are new assets introduced in the paper well documented and is the documentation provided alongside the assets?
    \item[] Answer: \answerNA{}.
    \item[] Justification: The paper does not introduce or release a new dataset, benchmark, pretrained model, or standalone software asset at submission time. The contribution is a private optimization method. The algorithmic details are documented in Section~\ref{sec:methodology}, and additional experimental and theoretical details are provided in Appendix~\ref{app:additional_experiments} and Appendix~\ref{app:sma_dp_sgd_theory}.

\item {\bf Crowdsourcing and research with human subjects}
    \item[] Question: For crowdsourcing experiments and research with human subjects, does the paper include the full text of instructions given to participants and screenshots, if applicable, as well as details about compensation, if any?
    \item[] Answer: \answerNA{}.
    \item[] Justification: The paper does not involve crowdsourcing, user studies, surveys, annotation tasks, or research with human subjects. All experiments are conducted on public machine-learning benchmark datasets.

\item {\bf Institutional review board (IRB) approvals or equivalent for research with human subjects}
    \item[] Question: Does the paper describe potential risks incurred by study participants, whether such risks were disclosed to the subjects, and whether Institutional Review Board approvals, or an equivalent approval/review based on the requirements of your country or institution, were obtained?
    \item[] Answer: \answerNA{}.
    \item[] Justification: The paper does not involve human subjects, crowdsourcing, private participant data, user studies, or human-subject interventions; therefore IRB approval or equivalent human-subjects review is not applicable.

\item {\bf Declaration of LLM usage}
    \item[] Question: Does the paper describe the usage of LLMs if it is an important, original, or non-standard component of the core methods in this research? Note that if the LLM is used only for writing, editing, or formatting purposes and does \emph{not} impact the core methodology, scientific rigor, or originality of the research, declaration is not required.
    \item[] Answer: \answerNA{}.
    \item[] Justification: LLMs are not used as an important, original, or non-standard component of the proposed method, theoretical analysis, experiments, or evaluation. Any use of language tools, if applicable, is limited to writing, editing, or formatting assistance and does not affect the scientific content or core methodology.

\end{enumerate}

\end{document}